\documentclass[twoside]{article}

\usepackage{etoc}

\usepackage[preprint]{aistats2027}
\usepackage[round]{natbib}

\usepackage[utf8]{inputenc}
\let\aistatsAnd\And               % amsmath redefines \And (used in \aistatsauthor)
\usepackage{amsmath,amssymb,amsthm}
\newtheorem*{recalledproposition}{Proposition \recallednumber}
\newtheorem*{recalledtheorem}{Theorem \recallednumber}

\newenvironment{restateproposition}[1]
  {\def\recallednumber{\ref{#1}}\begin{recalledproposition}}
  {\end{recalledproposition}}

\newenvironment{restatetheorem}[1]
  {\def\recallednumber{\ref{#1}}\begin{recalledtheorem}}
  {\end{recalledtheorem}}
\let\And\aistatsAnd
\usepackage{enumitem}            % enumerate[label=...,leftmargin=*]
\usepackage{algorithm}           % algorithm float (loads float -> [H])
\usepackage{algpseudocode}       % \Require, \State, \For, \Statex, ...
\usepackage{verbatim}            % comment environment

\theoremstyle{plain}
\newtheorem{theorem}{Theorem}
\newtheorem{proposition}{Proposition}
\newtheorem{corollary}{Corollary}
\newtheorem{lemma}{Lemma}

\theoremstyle{definition}
\newtheorem{definition}{Definition}

\newcommand{\Acal}{\mathcal{A}}

\newcommand{\Xcal}{\mathcal{X}}
\newcommand{\Zcal}{\mathcal{Z}}

\newcommand{\Mcal}{\mathcal{M}}
\newcommand{\Ocal}{\mathcal{O}}
\newcommand{\Vcal}{\mathcal{V}}

\newcommand{\eps}{\varepsilon}
\newcommand{\E}{\mathbb{E}}
\newcommand{\R}{\mathbb{R}}

\newcommand{\PP}{\mathbb{P}}

\DeclareMathOperator{\Var}{Var}

\usepackage[colorlinks=true,citecolor=blue,linkcolor=blue,urlcolor=blue]{hyperref}

\begin{document}
\etocdepthtag{main}


\twocolumn[

\aistatstitle{Minimax PAC Bounds for Learning in Exogenous Contextual MDPs}

\aistatsauthor{ Corentin Pla  \And  Hugo Richard  \And  Marc Abeille  \And Vianney Perchet}

\aistatsaddress{ CREST, ENSAE \\ Criteo AI Lab \\ FairPlay Joint Team \And  Criteo AI Lab \\ FairPlay Joint Team \And Criteo AI Lab \\ FairPlay Joint Team \And CREST, ENSAE \\ Criteo AI Lab \\ FairPlay Joint Team} 
 ]

\begin{abstract}
We introduce a PAC framework in which the learner can access sampling oracles both before and at decision time. Sample complexity is measured by a pair $(n,m)$, where $n$ is the learning budget spent before a query is known and $m$ is the additional sampling budget per query. We demonstrate its relevance in discounted Markov decision processes with exogenous i.i.d.\ contexts revealed before acting. Contexts may affect both rewards and transitions but remain uncontrolled by the agent. The learner can sample the unknown context distribution and the transition kernel. We study policy evaluation (PE), best-value estimation (BVE), and best-policy extraction (BPE). When rewards and transitions are known, a variance-reduced algorithm solves all three tasks with sample complexity
$\bigl(\widetilde O((1-\gamma)^{-3}\varepsilon^{-2}),0\bigr)$,
which is minimax optimal up to logarithmic factors. Let $\mathcal{X}$ be the controlled state space. When transitions are also unknown, we give a PE algorithm with complexity
$\bigl(\widetilde O(|\mathcal X|(1-\gamma)^{-3}\varepsilon^{-2}),
\widetilde O((1-\gamma)^{-2}\varepsilon^{-2})\bigr)$
and matching lower bounds at this budget pair. For BVE and BPE, we give an algorithm with a common offline budget
$\widetilde O(|\mathcal X|^2|\mathcal A|(1-\gamma)^{-4}\varepsilon^{-2})$
and respective query costs
$\widetilde O(|\mathcal A|(1-\gamma)^{-2}\varepsilon^{-2})$
and
$\widetilde O(|\mathcal A|(1-\gamma)^{-3}\varepsilon^{-2})$.
Importantly, all bounds are independent of the context-space cardinality.
\end{abstract}

\section{Introduction}

In planning, it is common to use simulations both before deployment and at each decision. Simulations collected in advance help learn how to act, while further simulations guide the choice in the current state. AlphaZero is a familiar example: it learns a policy and a value function from simulated self-play games, then runs additional search before each move \citep{silver2017masteringchessshogiselfplay}. The amount of simulation needed at decision time therefore matters alongside the amount used for learning. Computer science routinely distinguishes these two costs.
Consider searching for items in a list of size $N$. If the list is
unsorted, each search takes $O(N)$ comparisons. Sorting it once costs
$O(N\log N)$ comparisons, but reduces each subsequent lookup to $O(\log N)$
comparisons using binary search. The initial work makes later queries cheaper. This separation between preprocessing and query costs is central in computer science to the analysis of data structures \citep{thorup2005approximate}. We study the statistical counterpart of this trade-off: how sampling during an initial learning phase can reduce the number of additional samples needed to answer subsequent queries with PAC guarantees.

As a first setting in which to study this question, we consider discounted MDPs with exogenous i.i.d.\ contexts revealed before acting. The agent controls the state, while the context provides external information that may affect both rewards and transitions. This structure gives a tractable setting for studying offline and query-time sample costs, and is motivated by several applications. A canonical example is Tetris \citep{algorta2019gametetris,gabillon2013tetris}, where the controlled state is the board configuration and the context is the incoming piece. A similar distinction appears in online advertising \citep{badanidiyuru2021learningbidcontextualprice,NIPS2017_0bed45bd,abhishek2020designingtruthfulcontextualmultiarmed}: the advertiser manages a budget and observes the arriving user's profile before bidding. In ride-hailing \citep{hailing-1,bao2025timingmatchdeepreinforcement}, the controlled state is the fleet configuration, while the context is an incoming ride request, including its origin, destination, and timing. In e-commerce pricing \citep{ban2021personalized,zhao2026contextualdynamicpricing}, the controlled state includes inventory, while the context describes the customer or session observed before choosing a price. 
\subsection{Related work.}

The sample complexity of tabular reinforcement learning with a generative model is well understood. \cite{azar2012samplecomplexityreinforcementlearning} proved minimax PAC bounds for discounted MDPs, with optimal dependence \(\widetilde\Theta(|\mathcal X||\mathcal A|/((1-\gamma)^3\varepsilon^2))\). Near-optimal algorithms were then obtained through variance-reduced q-value iteration \citep{sidford2019nearoptimaltimesamplecomplexities} and model-based plug-in methods \citep{agarwal2020modelbasedreinforcementlearninggenerative}. For model-free methods, variance-reduced Q-learning is minimax optimal up to logs \citep{wainwright2019variancereducedqlearningminimaxoptimal}, whereas standard Q-learning is generally suboptimal, with tight \((1-\gamma)^{-4}\) dependence when \(|\mathcal A|\ge 2\) \citep{li2023qlearningminimaxoptimaltight}. Instance-dependent guarantees for policy evaluation and best-policy identification were also studied in generative and online settings \citep{pananjady2020instancedependentellinftyboundspolicyevaluation,marjani2021navigatingbestpolicymarkov}. We move from a single sample budget to separate offline and query budgets, studying this distinction in MDPs with exogenous i.i.d.\ contexts.

A complementary line exploits structural knowledge of the dynamics to improve sample efficiency. \citep{Alharbi_2024} study additive-disturbance MDPs with known transition function and i.i.d.\ unknown disturbances, obtaining bounds independent of state and action cardinalities. \citep{modi2019samplecomplexityreinforcementlearning} consider model-based RL with a given simulator ensemble, where the true MDP is approximated by state-dependent convex combinations of base transition and reward models, their complexity depends on the number of mixing parameters rather than on $|\Xcal|$ or $|\Acal|$.

Contextual and exogenous MDPs have been studied under different information and structural assumptions. \citet{hallak2015contextualmarkovdecisionprocesses} consider a latent context that determines rewards and transitions and remains fixed throughout each episode. In ExoMDPs,
\citep{efroni2022sampleefficientreinforcementlearningpresence} consider action-independent, reward-irrelevant exogenous variables, so the goal is to learn a policy that ignores them, while \citet{trimponias2026reinforcementlearningexogenousstates} exploit an additive reward decomposition to remove an exogenous reward component without changing optimal policies. \citet{chitnis2019learningcompactmodelsplanning} instead address model reduction, selecting a subset of exogenous variables that preserves decision quality while reducing planning costs. Other works exploit known endogenous dynamics: \citet{sinclair2023hindsightlearningmdpsexogenous} develop hindsight-learning methods for resource-management problems with known deterministic dynamics and rewards conditional on exogenous inputs, while \citet{wan2025exploitingexogenousstructuresampleefficient} derive regret guarantees through a linear-mixture representation, with complexity governed by the size or effective dimension of the exogenous input space. \citet{maran2026learningmarkovdecisionprocesses} obtain finite-horizon regret bounds for Markovian exogenous dynamics with a known endogenous transition kernel. We also allow unknown endogenous dynamics: even with i.i.d.\ contexts, answering a new state-context query can then require fresh transition samples, a difficulty absent under their known-kernel assumption.

Our framework also directly subsumes reinforcement learning with one-step transition lookahead (\cite{merlis2024reinforcementlearninglookaheadinformation,pla2026hardnessreinforcementlearningtransition,pla2026nearoptimalreinforcementlearningmultistep}). In that setting, before acting, the agent
observes the collection of next states that would be reached under the
available actions; this collection can be viewed as a single exogenous context,
sampled by the environment and revealed before the action. \cite{lu2025reinforcementlearningimperfecttransition} study MDPs with multi-step transition predictions and provide finite-sample guarantees for learning a prediction-aware Bayesian value function. Our results sharpen this sample-complexity analysis.
\subsection{Contribution.}
We introduce a new PAC learning framework in which sample complexity is measured by a pair $(n_{\rm learn},m_{\rm query})$: the number of oracle calls made before a query is known, and the additional number needed to answer it. As a first setting, we consider discounted MDPs with exogenous i.i.d.\ contexts revealed before acting. Rewards are known, and the learner has sampling access to the unknown context distribution $\mu$. The transition kernel $P$ is either known or accessible through a sampling oracle. Write $\beta=(1-\gamma)^{-1}$.When $P$ is known and only $\mu$ is unknown, a variance-reduced
algorithm solves policy evaluation (PE), best-value estimation (BVE),
and best-policy extraction (BPE) with complexity $(n_{\rm learn},m_{\rm query})
=
\left(
\widetilde O\!\left(\frac{\beta^3}{\eps^2}\right),0
\right)$ (Theorem \ref{thm:main-mu-unknown})
This rate is minimax optimal up to logarithmic factors. As a corollary,
we improve existing sample-complexity bounds for perfect one-step
transition look-ahead (Corrolary \ref{cor:lookahead-optimal}). When both $\mu$ and $P$ are unknown, we obtain a PE algorithm with
complexity $(n_{\rm learn},m_{\rm query})
=
\left(
\widetilde O\!\left(\frac{|\Xcal|\beta^3}{\eps^2}\right),
\widetilde O\!\left(\frac{\beta^2}{\eps^2}\right)
\right),$ and matching lower bounds at this budget pair. The offline phase learns
an averaged value on $\Xcal$, while fresh transition samples recover
the value at the queried $(x,z)$.

For BVE and BPE with unknown dynamics, we develop a new algorithm
based on multilevel Monte Carlo.  Both tasks share the offline budget $n_{\rm learn}
=
\widetilde O\!\left(
\frac{|\Xcal|^2|\Acal|\beta^4}{\eps^2}
\right),$ with respective query costs $m_{\rm query}^{\rm BVE}
=
\widetilde O\!\left(\frac{|\Acal|\beta^2}{\eps^2}\right),
\
m_{\rm query}^{\rm BPE}
=
\widetilde O\!\left(\frac{|\Acal|\beta^3}{\eps^2}\right)$ (Theorem \ref{thm:main-pe-Punknown}). For BPE, we use variance estimates to select actions. We prove that
$\Omega(\beta^3/\eps^2)$ samples per query are necessary in the worst
case, even with two states and two actions, when the offline budget
is independent of $|\Zcal|$ (Proposition \ref{prop:bpe-query-lower}). The optimal offline complexity for BVE and BPE
remains open. All our sample bounds are independent of the
context-space cardinality.

\section{Setting and objectives}
\label{sec:setting}
\begingroup
\setlength{\abovedisplayskip}{4pt}
\setlength{\belowdisplayskip}{4pt}
\setlength{\abovedisplayshortskip}{2pt}
\setlength{\belowdisplayshortskip}{4pt}
\subsection{Exogenous contextual Markov decision process.}
\label{subsec:model}
\paragraph{Model.} We consider finite discounted Markov decision processes with exogenous
i.i.d.\ contexts. A model is a tuple $\Mcal=(\Xcal,\Zcal,\Acal,P,r,\mu,\gamma)$, where \(\Xcal\) is the finite controlled state space, \(\Zcal\) is the finite, arbitrarily large, context
space, \(\Acal\) is the finite action space, \(P\) is a controlled transition kernel, \(r\) is a reward function, \(\mu\) is the context distribution, and
\(\gamma\in[0,1)\) is the discount factor. The transition kernel is
$ P:\Xcal\times\Acal\times\Zcal \to \Delta(\Xcal),$
so that \(P(\cdot\mid x,a,z)\) is the law of the next controlled state after
taking action \(a\) in controlled state \(x\) under context \(z\). The reward
function is $r:\Xcal\times\Acal\times\Zcal\to[0,1]$.
Throughout the paper, for simplicity, \(r\) is deterministic and known. Interactions are as follows: at each time \(t\ge 0\), the learner observes the augmented state \((X_t,Z_t)\in\Xcal \times \Zcal\), chooses an action \(A_t\in\Acal\), receives reward
\(r(X_t,A_t,Z_t)\), and transitions to the next state $(X_{t+1},Z_{t+1}) \sim P(\cdot\mid X_t,A_t,Z_t)\otimes\mu(\cdot)$. 

\paragraph{Context-aware policies and value function.} A deterministic stationary context-aware policy is a map
\(\pi:\Xcal\times\Zcal\to\Acal\), under which the action at time \(t\) is \(A_t=\pi(X_t,Z_t)\). Restricting to policies
that ignore \(Z_t\) would be suboptimal and would not capture the
decision problem of interest.
For a policy \(\pi\), define its value function at state-context $(x,z)$ by
$
V^\pi(x,z)
:=
\E^\pi\!\left[
\sum_{t=0}^{\infty}\gamma^t r(X_t,A_t,Z_t)
\,\middle|\,
X_0=x,\ Z_0=z
\right].
$
The optimal value is $V^\star(x,z):=\sup_{\pi}V^\pi(x,z)$
where the supremum is over deterministic stationary context-aware policies. The optimal value satisfies the following Bellman optimality equation :
\begin{align}
    &V^\star(x,z)=\notag \\&\max_{a\in\Acal}\{r(x,a,z)+\gamma
\sum_{x',z'}P(x'\mid x,a,z)\mu(z')V^\star(x',z')\}.
\label{eq:bellman}
\end{align}

\subsection{Learning objectives}
\label{sec:pac-objectives}
Our objective is to determine the sample complexity of policy evaluation (PE), best value estimation (BVE), and best-policy extraction (BPE), which are the tasks of estimating $V^\pi(x,z)$, for a given policy $\pi : \mathcal{X}\times\mathcal{Z}\to \Delta(\mathcal{A}) $, estimating $V^\star(x,z)$, and extracting a policy that is $\varepsilon$-optimal at $(x,z)$. The learner has access to the generative oracle : depending on the regime, an oracle call returns either a sample from the context distribution \(\mu\), or a transition sample generated from \(P\) at a chosen state-action-context triplet. Oracle calls may be made before the query point \((x,z)\) is known, during an offline learning phase, or after \((x,z)\) is revealed, when answering that query.

Formally, a learning algorithm is a pair \((\Acal_{\rm off},\Ocal)\). The offline algorithm \(\Acal_{\rm off}\), given access to the prescribed generative oracles, makes at most \(n_{\rm learn}\) oracle calls and outputs a query mechanism \(\Ocal \). Given a state-context pair $(x,z)$, the query mechanism $\Ocal(x,z)$ may make at most \(m_{\rm query}\) new generative-oracle calls to produce its output.
Depending on the objective, the output of \(\Ocal\) when evaluated at state-context pair $(x,z)$ is denoted by
\(
    \widehat{\Vcal}_{}(x,z)\in\R
\) for value-estimation tasks, and by
\(
    \widehat{\Pi}_{}(x,z)\in\Acal
\) for policy-extraction tasks.
\begin{definition}[Sample complexity]
\label{def:algorithmic-complexity}
A learning algorithm $(\Acal_{\rm off},\Ocal)$ has a sample complexity $(n_{\rm learn},m_{\rm query})$ if the offline algorithm $\Acal_{\rm off}$ uses at most
$n_{\rm learn}$ generative-oracle calls, and if, for every query point
$(x,z)$, the query-time procedure $\Ocal$ uses at most
$m_{\rm query}$ additional generative-oracle calls.
\end{definition}
\begin{definition}[Decision-time PAC objectives]
\label{def:pac-objectives}
Fix \(\eps>0\) and \(\delta\in(0,1)\). 
An algorithm satisfies an objective below if the corresponding
inequality holds for every fixed query point $(x,z)$. The probability
is taken over all randomness in the offline phase and at query time:
\begin{enumerate}[
    label=\rm(\roman*),
    leftmargin=*
]
\item \emph{Policy evaluation (PE).} $\forall (x,z)\in \mathcal{X}\times \mathcal{Z},$
\(
    \ \mathbb P\!\left(
    \bigl|\widehat{\Vcal}^{\pi}(x,z)-V^{\pi}(x,z)\bigr|>\eps
    \right)\le \delta .
\)
\item \emph{Best-value estimation (BVE).} $ \forall (x,z)\in \mathcal{X}\times \mathcal{Z},$
\(
    \ \mathbb P\!\left(
    \bigl|\widehat{\Vcal}^{\star}(x,z)-V^{\star}(x,z)\bigr|>\eps
    \right)\le \delta .
\)
\item \emph{Best-policy extraction (BPE).} $ \forall (x,z)\in \mathcal{X}\times \mathcal{Z},$
\(
    \ \mathbb P\!\left(
    |V^{\star}(x,z)-V^{\widehat\Pi}(x,z)|>\eps
    \right)\le \delta .
\)
\end{enumerate}
\end{definition}
For BPE, the first query to $(x,z)$ uses the offline output and
at most $m_{\rm query}$ new oracle calls to choose an action.
Every subsequent query to this pair returns the same action
without further sampling. Samples collected for other queries
are not reused. These choices define a deterministic stationary
policy $\widehat\Pi:\mathcal X\times\mathcal Z\to\mathcal A$.
For each resulting policy, $V^{\widehat\Pi}$ averages over
environment trajectories. The PAC probability is taken over
the offline samples and the query samples defining
$\widehat\Pi$ at all state-context pairs.
\par
\endgroup

\section{Learning the Context Distribution with Known Dynamics}
\label{sec:mu-unknown-P-known}

\subsection{Method}

We present the algorithm and the analysis for BPE below. Similar arguments apply to PE and BVE. Assume that \(P\) and \(r\) are known, and that only \(\mu\) is learned
from i.i.d.\ samples. Taking the expectation of both sides of the
augmented Bellman equation (\ref{eq:bellman}) with respect to \(Z\sim\mu\) gives
\begin{align}
&\underbrace{
  \E_{Z\sim\mu}[V^\star(x,Z)]
}_{\bar V^\star(x)}
=\E_{Z\sim\mu}\Bigl[
  \max_{a\in\Acal}\Bigl\{
    r(x,a,Z) \notag\\
&\qquad+\gamma\E_{x'\sim P(\cdot\mid x,a,Z)}
  \Bigl[
    \underbrace{
      \E_{Z'\sim\mu}[V^\star(x',Z')]
    }_{\bar V^\star(x')}
  \Bigr]
  \Bigr\}
\Bigr].
\end{align}
This defines a Bellman equation for the averaged value
\(\bar V^\star\) on \(\Xcal\). 
The key idea is to learn \(\hat{\bar{ V}}^\star\), an estimate of  \(\bar V^\star\),
rather than a value for every \((x,z)\). Then, once \(\hat{\bar{ V}}^\star\) has been computed one can compute the action at state $(x,z)$ using the greedy selector : 
\begin{align}
&\widehat\Pi(x,z) \notag 
\\&\in
\arg\max_{a\in\Acal}\Bigl\{r(x,a,z)+\gamma\sum_{x'}P(x'\mid x,a,z)
\hat{\bar V}^\star(x')\Bigr\}.
\label{eq:greedy}
\end{align} since \(P\) and \(r\) are known, these mechanisms use no additional
sampling at query time. Thus \(m_{\rm query}=0\), and the only cost to
report is the offline learning budget \(n_{\rm learn}\).

\subsection{Algorithm}
\label{sec:halferr-overview}

We now describe the offline algorithm for BPE. For any function
\(\bar V:\Xcal\to\R\), define the one-step Bellman backup
\[
g_{x,\bar V}^{\star}(z)
:=\max_{a\in\Acal}\Bigl\{r(x,a,z)
+\gamma\sum_{x'}P(x'\mid x,a,z)\bar V(x')\Bigr\},
\]
and the context-averaged Bellman operator
\[
(T\bar V)(x):=\E_{Z\sim\mu}[g_{x,\bar V}^{\star}(Z)].
\]
Since \(T\) is a \(\gamma\)-contraction, \(\bar V^\star\) is its unique
fixed point. Since \(P\) and \(r\) are known, \(g_{x,\bar V}^\star(Z)\) can be
computed exactly on every sampled context. Thus the only statistical
task is to approximate the expectation over \(Z\sim\mu\). Given samples
\(Z_1,\ldots,Z_n\sim\mu\), the empirical operator is
\[
(\hat T\bar V)(x):=\frac1n\sum_{i=1}^n g_{x,\bar V}^\star(Z_i).
\]

Our algorithm adapts the variance-reduced value iteration scheme of
\citet{sidford2019nearoptimaltimesamplecomplexities}. It estimates one
reference Bellman backup, then updates it using estimates of small
differences. Given a $u$-subsolution satisfying
\[
0\le\bar V^{(0)}\le T\bar V^{(0)},
\qquad
\|\bar V^\star-\bar V^{(0)}\|_\infty\le u,
\]
the routine \textsc{HalfErr} returns a $(u/2)$-subsolution with high
probability. Starting from zero, apply \textsc{HalfErr}
$K=\lceil\log_2(\beta/\eps)\rceil$ times, with confidence
$\delta/(2K)$ at each call. The resulting estimate
$\hat{\bar V}^\star$ is used in the greedy selector~\eqref{eq:greedy}.
The full procedure and confidence corrections are given in
Appendix~\ref{proof:mu-known-ub-star}.

\begin{algorithm}[H]
\caption{\textsc{HalfErr}}
\label{alg:halferr-main}
\begin{algorithmic}[1]
\Require A $u$-subsolution $\bar V^{(0)}$; confidence $\delta$.
\Ensure A $(u/2)$-subsolution with probability at least $1-2\delta$.
\State Choose $R=\widetilde O(\beta)$,
$n_1=\widetilde O(\beta^3/u^2+\beta^2)$ and
$n_2=\widetilde O(\beta^2)$.
\State Compute a lower confidence estimate $w$ of $T\bar V^{(0)}$
using $n_1$ context samples.
\For{$i=1,\ldots,R$}
  \State Compute a lower confidence estimate $\Delta^{(i)}$ of
  $T\bar V^{(i-1)}-T\bar V^{(0)}$ using $n_2$ fresh contexts,
  shared between the two terms.
  \State $\bar V^{(i)}\gets
  \max\{\bar V^{(i-1)},w+\Delta^{(i)}\}$, coordinatewise.
\EndFor
\State \Return $\bar V^{(R)}$.
\end{algorithmic}
\end{algorithm}

\subsection{Guarantees}
\label{sec:main-results-mu-unknown}

We now give the sample-complexity bounds and sketch their proofs.

\begin{proposition}[Minimax lower bound]
\label{prop:lb-main}
Fix $\gamma\in[1/2,1)$ and $\eps\in(0,\beta/18)$.
For each objective
$\mathcal P\in\{\mathrm{PE},\mathrm{BVE},\mathrm{BPE}\}$
from Definition~\ref{def:pac-objectives}, there exists a family
of contextual MDPs differing only in the context distribution $\mu$
such that any algorithm satisfying $\mathcal P$ with confidence
$\delta=1/6$ on every MDP in this family must use
\[
n_{\rm learn}+m_{\rm query}
\ge c\,\frac{\beta^3}{\eps^2},
\]
where $c>0$ is a universal constant.
\end{proposition}

\begin{proof}[Proof sketch]
We construct a four-state MDP. At $x_0$, actions $a_1$ and $a_2$
lead respectively to $x_1$ and $x_2$, where the reward is one until
absorption in a zero-reward state $x_T$. Context $2$ causes absorption
from $x_1$, while context $1$ causes absorption from $x_2$ with
probability $\rho/(1-\rho)$, where $\rho=1/(2\gamma\beta)$.
Otherwise, the process stays in the same state. Actions have no
effect after leaving $x_0$.

The two actions have equal values when $\mu(1)=1-\rho$.
We compare two context distributions with
$\mu(1)=1-\rho-\Delta$ and $\mu(1)=1-\rho+\Delta$.
The optimal action at $(x_0,1)$ differs between these instances,
and the action gap is of order $\beta^2\Delta$.
Choosing $\Delta\asymp\eps/\beta^2$ with a sufficiently large
constant ensures that no policy is $\eps$-optimal in both instances.
Thus BPE requires distinguishing the two context distributions.

Each oracle call provides one i.i.d.\ sample from $\mu$, and the
per-sample divergence is $O(\beta\Delta^2)$.
Le Cam's two-point method therefore gives
\[
n_{\rm learn}+m_{\rm query}
\gtrsim\frac{1}{\beta\Delta^2}
\asymp\frac{\beta^3}{\eps^2}.
\]
The full proof is given in Appendix~\ref{subsec:lb-mdp}.
\end{proof}

This lower bound is matched, up to logarithmic factors, by the following
upper bound.

\begin{theorem}
\label{thm:main-mu-unknown}
Assume that \(P\) and \(r\) are known and that \(\mu\) is accessible through
i.i.d.\ samples. For every \(\eps\in(0,1]\) and \(\delta\in(0,1)\), each of
PE, BVE, and BPE admits a decision-time mechanism with complexity
\[
(n_{\rm learn},m_{\rm query})
=
\left(
\widetilde O\!\left(\frac{\beta^3}{\eps^2}\right),
0
\right),
\]
satisfying the corresponding PAC objective of
Definition~\ref{def:pac-objectives}.
\end{theorem}

\begin{proof}[Proof sketch.]
    We follow the concentration and error-halving arguments of
\citet{sidford2019nearoptimaltimesamplecomplexities}. The main
adaptation is to control the variance of the contextual Bellman
backup. We split each martingale increment of the discounted return
into two parts: the observation of the context and the next-state
transition. Their orthogonality bounds the accumulated contextual
variance by the variance of the return. Writing
$\bar P^{\pi^\star}(x'\mid x)
:=\E_{Z\sim\mu}[P(x'\mid x,\pi^\star(x,Z),Z)]$,
this gives
\[
\left\|
(I-\gamma\bar P^{\pi^\star})^{-1}
\sqrt{\Var_{Z\sim\mu}
(g^\star_{\cdot,\bar V^\star}(Z))}
\right\|_\infty
\lesssim\beta^{3/2}.
\]
Together with the stability estimate of
Lemma~\ref{lem:sqrt-stab}, this yields the error-halving guarantee
and a total budget $\widetilde O(\beta^3/\eps^2)$.

For BPE, we must also pass from the averaged estimate to a guarantee
at each $(x,z)$. The subsolution property ensures that the greedy
policy's value is at least
$g^\star_{x,\hat{\bar V}^\star}(z)$. Hence
\begin{align}
    0&\le V^\star(x,z)-V^{\widehat\Pi}(x,z)
\le\gamma\|\bar V^\star-\hat{\bar V}^\star\|_\infty
\le\eps.
\end{align}

Thus policy extraction requires no additional factor $\beta$.
The arguments for PE and BVE are similar, and all three objectives
require no sampling at query time. The full proof is given in
Appendix~\ref{proof:mu-known-ub-star}.
\end{proof}

\subsubsection{One-step perfect transition lookahead}
\label{sec:1-step-look-ahead-application}

Consider perfect one-step transition lookahead
\citep{pla2026hardnessreinforcementlearningtransition,
merlis2024reinforcementlearninglookaheadinformation}
in a finite discounted MDP
$\Mcal_0=(\Xcal,\Acal,P_0,r,\gamma)$,
with known rewards and unknown $P_0$.
Before each action, the agent observes a fresh transition table
$Z_t\sim\mu$, where
\[
\Zcal=\Xcal^{\Xcal\times\Acal},
\qquad
\mu=\bigotimes_{(x,a)\in\Xcal\times\Acal}
P_0(\cdot\mid x,a).
\]
Thus $Z_t(x,a)$ specifies the successor of every state-action
pair. After choosing $A_t$, the agent receives $r(X_t,A_t)$
and moves to $X_{t+1}=Z_t(X_t,A_t)$.
The tables are independent across time.

\begin{proposition}
\label{prop:exo-look}
This lookahead model is equivalent to the exogenous contextual
MDP with context distribution $\mu$ and
\[
P(x'\mid x,a,z)=\mathbf1\{x'=z(x,a)\},
\qquad
\tilde r(x,a,z)=r(x,a).
\]
Corresponding policies have identical values in both models.
\end{proposition}

Although $P_0$ is unknown, the contextual transition kernel
$P$ is known: once $z$ is observed, the successor is given
by a table lookup. The unknown $P_0$ enters only through $\mu$.
The reduction is proved in Appendix~\ref{proof:exo-look};
Theorem~\ref{thm:main-mu-unknown} therefore applies directly.

\begin{corollary}[Perfect one-step transition lookahead]
\label{cor:lookahead-optimal}
For every $\eps\in(0,1]$ and $\delta\in(0,1)$,
each of PE, BVE, and BPE admits a decision-time mechanism
satisfying Definition~\ref{def:pac-objectives} with
\[
(n_{\rm learn},m_{\rm query})
=
\left(
\widetilde O\!\left(\frac{\beta^3}{\eps^2}\right),0
\right),
\]
where $n_{\rm learn}$ counts full transition-table samples.
\end{corollary}

No additional sampling is required once the query table is
observed. Generating each offline table through a tabular
generative model costs $|\Xcal||\Acal|$ individual transitions,
giving a total of
$\widetilde O(|\Xcal||\Acal|\beta^3/\eps^2)$.
For estimating the averaged optimal value, our table-sample
bound improves by a factor $\beta$ on the
$\widetilde O(\beta^4/\eps^2)$ bound obtained by specializing
\citet[Theorem~5.1]{lu2025reinforcementlearningimperfecttransition}
to perfect one-step predictions for all actions.

\section{Learning the Context Distribution and the Dynamics}
\label{sec:mu-P-unknown}

\subsection{Query mechanisms}
\label{subsec:decision-time-lift-Punknown}
\begingroup
\setlength{\abovedisplayskip}{4pt}
\setlength{\belowdisplayskip}{4pt}
\setlength{\abovedisplayshortskip}{2pt}
\setlength{\belowdisplayshortskip}{4pt}

We now consider the regime where both $\mu$ and $P$ are unknown,
while $r$ remains known. As in the known-dynamics case, the offline
phase estimates $\bar V^\pi$ for PE and
$\bar V^\star(x)=\E_{Z\sim\mu}[V^\star(x,Z)]$ for BVE and BPE.
These are the fixed points of $T^\pi$ and $T$, respectively.
BVE and BPE therefore share the same offline learning phase. The query mechanisms combine these offline estimates with samples
from $P$ at the requested state-context pair. For a continuation
vector $\bar V$, write
\[
q_{\bar V}(x,a,z)
=
r(x,a,z)+\gamma
\E_{X'\sim P(\cdot\mid x,a,z)}[\bar V(X')].
\]
Given a query $(x,z)$, draw
$X_{a,1},\ldots,X_{a,m}\sim P(\cdot\mid x,a,z)$
for each required action and compute
$$\widehat q_a
=
r(x,a,z)+\frac{\gamma}{m}
\sum_{j=1}^m\bar V(X_{a,j}).$$

\paragraph{Policy evaluation.}
Use $\bar V=\hat{\bar V}^\pi$, draw
$m=\widetilde O(\beta^2/\eps^2)$ successors for action
$\pi(x,z)$, and return
\[
\widehat{\mathcal V}^\pi(x,z)
=
\widehat q_{\pi(x,z)}.
\]

\paragraph{Best-value estimation.}
Use $\bar V=\hat{\bar V}^\star$, draw
$m=\widetilde O(\beta^2/\eps^2)$ successors per action,
and return
\[
\widehat{\mathcal V}^\star(x,z)
=
\max_{a\in\Acal}\widehat q_a.
\]
\paragraph{Best-policy extraction.}
For BPE, the value of the extracted policy depends on the actions
chosen at the initial query and at every subsequent step.
We must therefore control how action-selection errors accumulate
along its trajectory. On the first query to $(x,z)$, draw
$m=\widetilde O(\beta^3/\eps^2)$ successors per action.
Using $\bar V=\hat{\bar V}^\star$, compute $\widehat q_a$
and the sample variance $\widehat s_a^2$ of
$\gamma\hat{\bar V}^\star(X')$. Choose
\begin{equation}
\widehat\Pi(x,z)\in\arg\max_{a\in\Acal}
\left\{
\widehat q_a
+c\sqrt{
\frac{\widehat s_a^2}{m}
\log\frac{|\mathcal A||\mathcal X|\beta}{\delta\eps}
}
\right\},
\label{eq:bpe-query-Punknown}
\end{equation}
where $c$ is a sufficiently large universal constant and ties
are broken deterministically. Subsequent queries to the same
pair return the same action without further sampling.
The linear term in the empirical-Bernstein confidence bound
is common to all actions and does not affect the maximizer. This optimistic rule bounds the selection loss by the uncertainty
of the selected action. The resulting variance terms therefore
correspond to transitions followed by the extracted policy,
allowing a discounted variance bound to control their accumulation.
The query cost is
$\widetilde O(|\mathcal A|\beta^3/\eps^2)$.

\subsection{Offline learning}

\paragraph{Policy evaluation.}
For PE, the Bellman backup contains no maximum. Replacing its
transition expectation by a sampled successor therefore preserves
unbiasedness. For a fixed $\bar V$, draw $Z\sim\mu$ and
$X^x\sim P(\cdot\mid x,\pi(x,Z),Z)$, and use $\widehat g^\pi_{x,\bar V}
=
r(x,\pi(x,Z),Z)+\gamma\bar V(X^x).$
Taking the expectation over both draws gives
$\E[\widehat g^\pi_{x,\bar V}]=(T^\pi\bar V)(x)$.
We adapt the variance-reduced halving scheme of
\citet{sidford2019nearoptimaltimesamplecomplexities},
described in Section~\ref{sec:halferr-overview}, to these
sampled backups.
Each sampled context is shared across states and completed
by one transition sample per state.
The two terms in each correction use the same context and successor.

\paragraph{Best-value estimation and best-policy extraction.}
Unlike PE, replacing transition expectations by sample averages
does not give an unbiased Bellman backup. Even if each sampled
action value is unbiased, for a fixed $\bar V$,
\[
\E\!\left[\max_a\widehat q_a\mid Z\right]
\ge
\max_a q_{\bar V}(x,a,Z).
\]
The maximum can select an action because its estimate is too large,
particularly when several actions have similar values.
A direct approach uses large transition batches for every sampled
context to make this bias small. We adapt the multilevel Monte Carlo construction of
\citet{giles2019decision} to estimate the Bellman operator
uniformly over continuation vectors
$\bar V\in[0,\beta]^{|\mathcal X|}$, with high probability.
The multilevel idea is to estimate the expectation of a
large-batch backup using many cheap small-batch backups and
fewer samples of the differences between successive batch sizes.
These differences have decreasing variance, so they require
fewer repetitions even though each repetition is more expensive. A sample at level $\ell\ge0$ is constructed by drawing one
context $Z\sim\mu$ and, for every $(x,a)$, drawing
$2^\ell$ independent successors
$X_j^{x,a}\sim P(\cdot\mid x,a,Z)$.
Define
\[
G_\ell^x(\bar V)
=
\max_{a\in\Acal}
\left\{
r(x,a,Z)
+\frac{\gamma}{2^\ell}
\sum_{j=1}^{2^\ell}\bar V(X_j^{x,a})
\right\}.
\]
For $\ell\ge1$, split the same transition batch into two halves.
Let $G_{\ell-1}^{x,(1)}$ and $G_{\ell-1}^{x,(2)}$ denote the
backups computed from these halves, using the same context.
Set
\begin{align*}
D_0^x(\bar V)
&=G_0^x(\bar V),\\
D_\ell^x(\bar V)
&=
G_\ell^x(\bar V)
-\tfrac12G_{\ell-1}^{x,(1)}(\bar V)
-\tfrac12G_{\ell-1}^{x,(2)}(\bar V).
\end{align*}
The full-batch action estimates are the averages of the two
half-batch estimates. Consequently, $D_\ell^x(\bar V)\le0$
for $\ell\ge1$, with equality whenever the same action
maximizes both half-batch backups.
This construction satisfies $\Var(D_\ell^x(\bar V))
=
\widetilde O(\beta^2 2^{-\ell}),$ uniformly over $\bar V\in[0,\beta]^{|\mathcal X|}$,
without an action-gap assumption.
Moreover, the expectations telescope:
\[
\sum_{\ell=0}^L\E[D_\ell^x(\bar V)]
=
\E[G_L^x(\bar V)].
\]
Thus estimating the expectation at each level approximates
the expected large-batch backup, while the decreasing variance
allows fewer samples at the expensive levels. Bellman-update errors accumulate over the discounted horizon,
so we set $\eta=\eps/(16\beta)$. Choose $L=
\left\lceil
\log_2
\frac{c\beta^2\log(2|\mathcal A|)}{\eta^2}
\right\rceil,$ where $c$ is sufficiently large, so that
\[
0\le
\E[G_L^x(\bar V)]-(T\bar V)(x)
\le\eta/2
\]
for every $x$ and $\bar V\in[0,\beta]^{|\mathcal X|}$. To estimate each expectation $\E[D_\ell^x(\bar V)]$,
we use a median-of-means estimator.
At level $\ell$, repeat the sampling procedure independently
$Bq_\ell$ times, with independent samples across levels.
Divide these repetitions into $B$ groups of $q_\ell$ repetitions.
Write $D_{\ell,b,j}^x(\bar V)$ for the value computed from
repetition $j$ in group $b$.
The sampled contexts and successors are stored so that
each $D_{\ell,b,j}^x$ can be evaluated at different
continuation vectors.

For fixed $x$ and $\bar V$, average the values within each
group, then take the median of the $B$ averages.
Averaging reduces the variance by a factor $q_\ell$.
Once each average has the required accuracy with probability
greater than $1/2$, the median fails only if a majority of the
averages are inaccurate. Independence makes this probability
decrease exponentially with $B$.
Summing the level estimates gives
\[
(\widehat T\bar V)(x)
=
\sum_{\ell=0}^L
\operatorname{median}_{1\le b\le B}
\left\{
\frac1{q_\ell}
\sum_{j=1}^{q_\ell}D_{\ell,b,j}^x(\bar V)
\right\}.
\]
We take $B=\widetilde O\!\left(
|\mathcal X|+\log(1/\delta)
\right),
\qquad
q_\ell=\widetilde O\!\left(
1+\frac{\beta^2 2^{-\ell}}{\eta^2}
\right),$
Up to logarithmic factors and rounding,
$q_\ell 2^\ell$ is of order $\beta^2/\eta^2$:
the decreasing number of repetitions compensates for thevincreasing transition cost of each repetition.

These choices ensure that, with probability at least
$1-\delta/3$,
\begin{equation}
\sup_{\bar V\in[0,\beta]^{|\mathcal X|}}
\|\widehat T\bar V-T\bar V\|_\infty
\le\eta.
\label{eq:uniform-backup-Punknown}
\end{equation}
The proof first establishes concentration on a finite grid
of continuation vectors, then extends it to the whole cube
using Lipschitz continuity.
The logarithm of the grid size accounts for the additional
dependence on $|\mathcal X|$ in $B$; the grid is used only
in the analysis.
Because the guarantee holds simultaneously for every $\bar V$,
the same stored samples can be reused throughout value iteration,
even though the iterates depend on those samples.
The total number of context and transition calls is  $\widetilde O\!\left(
|\mathcal X||\mathcal A|
\bigl(|\mathcal X|+\log(1/\delta)\bigr)
\frac{\beta^2}{\eta^2}
\right).$
Starting from $\bar V^{(0)}=0$, perform
$J=\lceil\beta\log(8\beta/\eps)\rceil$ updates:
for $k=0,\ldots,J-1$,
\begin{equation}
\bar V^{(k+1)}
=
\min\left\{
\beta\mathbf1,\,
\max\left\{
\bar V^{(k)},\,
\widehat T\bar V^{(k)}-\eta\mathbf1
\right\}
\right\},
\label{eq:multilevel-iteration-Punknown}
\end{equation}
where both operations are coordinatewise.
These iterations require no additional oracle calls.
On the event~\eqref{eq:uniform-backup-Punknown},
$T\bar V^{(k)}-2\eta\mathbf1
\le
\widehat T\bar V^{(k)}-\eta\mathbf1
\le
T\bar V^{(k)}.$ Subtracting $\eta$ therefore keeps the estimated backup below
the true backup, while taking the maximum with $\bar V^{(k)}$
makes the iterates nondecreasing.
The updates preserve the subsolution property
$0\le\bar V^{(k)}\le T\bar V^{(k)}$, and contraction gives
\[
0\le\bar V^\star-\bar V^{(J)}
\le
\bigl(\gamma^J\beta+2\beta\eta\bigr)\mathbf1
\le\frac{\eps}{4}\mathbf1.
\]
Return $\hat{\bar V}^\star=\bar V^{(J)}$.
This offline estimate is shared by BVE and BPE, which use
the respective query mechanisms defined above.
The algorithm is formally described in
Appendix~\ref{app:unknown-control-upper}.

\par
\endgroup
\subsection{Results and proof sketches}
\label{subsec:main-results-Punknown}

\begin{theorem}[Unknown $\mu$ and $P$]
\label{thm:main-pe-Punknown}
Assume known rewards and sampling access to $\mu$ and $P$.
For $\gamma\in[1/2,1)$, $\eps\in(0,1]$, and
$\delta\in(0,1/4)$, the algorithms above satisfy the objectives
of Definition~\ref{def:pac-objectives} with complexities
\[
\begin{array}{c|cc}
& n_{\rm learn} & m_{\rm query}\\ \hline
\mathrm{PE}
& \widetilde O(|\mathcal X|\beta^3/\eps^2)
& \widetilde O(\beta^2/\eps^2)\\[0.3em]
\mathrm{BVE}
& \widetilde O(|\mathcal X|^2|\mathcal A|\beta^4/\eps^2)
& \widetilde O(|\mathcal A|\beta^2/\eps^2)\\[0.3em]
\mathrm{BPE}
& \widetilde O(|\mathcal X|^2|\mathcal A|\beta^4/\eps^2)
& \widetilde O(|\mathcal A|\beta^3/\eps^2)
\end{array}
\]
All bounds are independent of $|\Zcal|$.
\end{theorem}

\begin{proof}[Proof sketch (BPE)]
The offline sample cost follows from the sample sizes specified
above. Condition on the event~\eqref{eq:uniform-backup-Punknown},
on which $\hat{\bar V}^\star\le T\hat{\bar V}^\star$ and
$\|\bar V^\star-\hat{\bar V}^\star\|_\infty\le\eps/4$.
The subsolution property bounds the policy loss by the offline
approximation error plus the expected discounted sum of
action-selection losses along the extracted policy's trajectory.

At $(x,z)$, the local selection loss is
\[
D(x,z)
=
\max_a q_{\hat{\bar V}^\star}(x,a,z)
-
q_{\hat{\bar V}^\star}(x,\widehat\Pi(x,z),z).
\]
With $m$ successors per action, the empirical-Bernstein
rule~\eqref{eq:bpe-query-Punknown} gives, on its local
confidence event,
\[
D(x,z)
=
\widetilde O\!\left(
\frac{\gamma}{\sqrt m}
\sqrt{
\Var_{P(\cdot\mid x,\widehat\Pi(x,z),z)}
\bigl(\hat{\bar V}^\star(X')\bigr)
}
+\frac{\beta}{m}
\right).
\]
Using only the range $[0,\beta]$ would give an accumulated
error of order $\widetilde O(\beta^2/\sqrt m)$, requiring
$m=\widetilde O(\beta^4/\eps^2)$.
Because the variance above corresponds to the selected action,
a discounted variance argument instead bounds the contribution
from contexts where the local confidence bounds hold by $\widetilde O\!\left(
\frac{\beta^{3/2}}{\sqrt m}
+\frac{\beta^2}{m}
\right).$ Choose the local failure probability to be a sufficiently small
constant multiple of $\delta\eps/(|\mathcal X|\beta^2)$.
Markov's inequality and a union bound over controlled states
then ensure that, at every state, the contexts where the local
bound fails have a mass under $\mu$ at most a sufficiently small multiple
of $\eps/\beta^2$, except with probability $\delta/3$.
Since $D(x,z)\le\beta$, their discounted contribution is a
sufficiently small multiple of $\eps$.
The initial fixed query is controlled separately,
no union bound over contexts is needed. Taking $m=\widetilde O(\beta^3/\eps^2)$ and combining these
bounds with the offline error gives loss at most $\eps$
with probability at least $1-\delta$, using
$|\mathcal A|m$ query samples.
\end{proof}

\begingroup
\setlength{\abovedisplayskip}{4pt}
\setlength{\belowdisplayskip}{4pt}
\setlength{\abovedisplayshortskip}{2pt}
\setlength{\belowdisplayshortskip}{4pt}
\begin{proposition}[BPE query lower bound]
\label{prop:bpe-query-lower}
For every $\beta\ge256$, $\eps\in(0,1]$, and integer $M\ge1$,
there exists a family with two controlled states, two actions,
and $M$ contexts, sharing a known uniform context distribution,
such that any stationary decision-time mechanism satisfying
the pointwise BPE objective on this family with confidence
$\delta=1/6$ must satisfy
\[
\frac{n_{\rm learn}}{M}+m_{\rm query}
\ge c\frac{\beta^3}{\eps^2},
\]
where $c>0$ is a universal constant.
\end{proposition}

\begin{proof}[Proof sketch]
There is one live state $s$ with reward one and one zero-value
absorbing state. Each context has an independent binary label
indicating which action has self-loop probability $p+\Delta$;
the other action has probability $p-\Delta$, where
\[
p=1-\frac1\beta,
\qquad
\Delta\asymp\frac{\eps}{\beta^2}.
\]
Let $f_\Pi$ be a stationary policy's average probability of
choosing the worse action across contexts.
Its averaged self-loop probability is
$p+\Delta-2\Delta f_\Pi$, which gives
\[
V^\star(s,z)-V^\Pi(s,z)
\gtrsim
\Delta\beta^2 f_\Pi
\]
for every initial context $z$.
Thus errors across contexts reduce the continuation value
from every initial query.
With a sufficiently large constant in $\Delta$, the pointwise
PAC guarantee forces $\E[f_{\widehat\Pi}]$ to be bounded
away from $1/2$.

The offline budget provides at most $n_{\rm learn}/M$
informative samples per context on average, while each query
allows at most $m_{\rm query}$ additional samples.
Under independent uniform labels, each informative transition
has divergence $O(\beta\Delta^2)$ between the two possible labels.
Averaging the adaptive binary-testing bounds over contexts gives
\[
\E[f_{\widehat\Pi}]
\ge
\frac12\left(
1-C\sqrt{
\beta\Delta^2
\left(
\frac{n_{\rm learn}}{M}+m_{\rm query}
\right)
}
\right)
\]
for a universal constant $C$.
Combining this inequality with the PAC requirement and
$\Delta\asymp\eps/\beta^2$ yields the claimed lower bound.
\end{proof}
When the offline budget is independent of context cardinality,
this establishes the optimal dependence on $\beta$ and $\eps$
of the worst-case BPE query cost, up to logarithmic factors,
for fixed state and action cardinalities and constant confidence. Under the same restriction, the PE lower bounds proved Appendix \ref{sec:lower-bound-p-unkn} gives $m_{\rm query}
=\Omega(\beta^2/\eps^2),
\qquad
n_{\rm learn}+|\mathcal X|m_{\rm query}
=\Omega(|\mathcal X|\beta^3/\eps^2).$ The query lower bound matches our PE upper bound up to
logarithmic factors. The offline bound also matches whenever
$m_{\rm query}$ is at most a sufficiently small constant
multiple of $\beta^3/\eps^2$.
The optimal offline complexity for BVE and BPE remains open.

\par
\endgroup
\section{Conclusion}
\label{sec:conclusion}

We introduce a PAC criterion with separate budgets for offline
learning and sampling at query time. An algorithm is evaluated
both by the samples needed before a query is known and by the
additional samples needed to answer it with prescribed accuracy
and confidence. This formulation allows us to study how learning
in advance reduces the cost of subsequent queries, and which
sampling costs remain unavoidable.

Discounted MDPs with exogenous i.i.d.\ contexts provide a concrete
setting for this criterion. Learning an averaged value on $\Xcal$
yields sample bounds independent of $|\Zcal|$ for PE, BVE, and BPE.
With known dynamics, we obtain minimax rates without query-time
sampling. The same analysis improves the sample complexity of
perfect one-step transition lookahead by a factor $\beta$
\citep{lu2025reinforcementlearningimperfecttransition}.

With unknown dynamics, fresh transition samples complement the
offline estimate. For PE, our offline/query budget pair matches
lower bounds. For BVE and BPE, a new multilevel algorithm controls
the bias introduced by taking the maximum of sampled action values.
These two tasks share the same offline computation but have
different query costs, with quadratic and cubic dependence on
$\beta$, respectively. For BPE, the cubic dependence is unavoidable
even with two states and two actions when the offline budget is
independent of the number of contexts. 

The full offline/query tradeoff and the optimal offline complexity
of BVE and BPE remain open within our current setting.
Beyond these statistical questions, a further direction is to
relax the i.i.d.\ assumption on contexts. Extending the analysis
to temporally correlated contexts would require accounting for
the information that the current context provides about future
ones.

\subsection*{Acknowledgments}
This work was supported in part by the French National Research Agency (ANR) through grants ANR-20-CE23-0007 and ANR-23-CE23-0002, and through the PEPR IA FOUNDRY project (ANR23-PEIA-0003). It was also supported by Hi! PARIS and the ANR/France 2030 program (ANR23-IACL-0005).

\subsection*{AI Use Statement}

We used generative AI tools to assist with drafting and revising the
manuscript, including the writing and presentation of mathematical
proofs. The theoretical contributions and key proof ingredients were
developed by the authors; generative AI tools helped articulate and
refine these arguments. The authors reviewed all AI-assisted content
and checked the mathematical arguments for correctness. We take full
responsibility for the final content of this work, including its
text, mathematical claims, and proofs.

\bibliography{biblio}
\bibliographystyle{plainnat}

%%%%%%%%%%%%%%%%%%%%%%%%%%%%%%%%%%%%%%%%%%%%%%%%%%%%%%%%%%%%
\section*{Checklist}

\begin{enumerate}

  \item For all models and algorithms presented, check if you include:
  \begin{enumerate}
    \item A clear description of the mathematical setting, assumptions, algorithm, and/or model. [Yes]
    \item An analysis of the properties and complexity (time, space, sample size) of any algorithm. [Yes]
    \item (Optional) Anonymized source code, with specification of all dependencies, including external libraries. [Not Applicable]
  \end{enumerate}

  \item For any theoretical claim, check if you include:
  \begin{enumerate}
    \item Statements of the full set of assumptions of all theoretical results. [Yes]
    \item Complete proofs of all theoretical results. [Yes]
    \item Clear explanations of any assumptions. [Yes]
  \end{enumerate}

  \item For all figures and tables that present empirical results, check if you include:
  \begin{enumerate}
    \item The code, data, and instructions needed to reproduce the main experimental results (either in the supplemental material or as a URL). [Not Applicable]
    \item All the training details (e.g., data splits, hyperparameters, how they were chosen). [Not Applicable]
    \item A clear definition of the specific measure or statistics and error bars (e.g., with respect to the random seed after running experiments multiple times). [Not Applicable]
    \item A description of the computing infrastructure used. (e.g., type of GPUs, internal cluster, or cloud provider). [Not Applicable]
  \end{enumerate}

  \item If you are using existing assets (e.g., code, data, models) or curating/releasing new assets, check if you include:
  \begin{enumerate}
    \item Citations of the creator If your work uses existing assets. [Not Applicable]
    \item The license information of the assets, if applicable. [Not Applicable]
    \item New assets either in the supplemental material or as a URL, if applicable. [Not Applicable]
    \item Information about consent from data providers/curators. [Not Applicable]
    \item Discussion of sensible content if applicable, e.g., personally identifiable information or offensive content. [Not Applicable]
  \end{enumerate}

  \item If you used crowdsourcing or conducted research with human subjects, check if you include:
  \begin{enumerate}
    \item The full text of instructions given to participants and screenshots. [Not Applicable]
    \item Descriptions of potential participant risks, with links to Institutional Review Board (IRB) approvals if applicable. [Not Applicable]
    \item The estimated hourly wage paid to participants and the total amount spent on participant compensation. [Not Applicable]
  \end{enumerate}

\end{enumerate}

\clearpage
\appendix

\thispagestyle{empty}

% Supplementary material: To improve readability, you must use a single-column format for the supplementary material.
\onecolumn
\clearpage

\begingroup
\etocsetnexttocdepth{3}
\etocsettocstyle{\section*{Appendix}}{}
\etocsettagdepth{main}{none}
\etocsettagdepth{appendix}{all}
\tableofcontents
\endgroup

\clearpage
\etocdepthtag{appendix}

\section{Proofs : \texorpdfstring{$\mu$}{μ} unknown, \texorpdfstring{$P$}{P} known}
\label{sec:lb-worst-case}
\subsection{Lower bounds}
\label{subsec:lb-mdp}
\subsubsection{Policy Evaluation}
\label{subsec:lb-pe-known}
\setcounter{theorem}{0}
\setcounter{proposition}{0}

\begin{restateproposition}{prop:lb-main}[PE]
Fix $\gamma\in[1/2,1)$, let $\beta=(1-\gamma)^{-1}$, and take
$\eps\in(0,\beta/18)$.
There exist two contextual MDPs with two states, one action, and two
contexts, sharing the same known transition kernel and reward function,
such that any algorithm satisfying the PE objective of
Definition~\ref{def:pac-objectives} on both instances with confidence
$\delta=1/6$ must satisfy
\[
    n_{\rm learn}+m_{\rm query}
    \ge c_{\rm PE}\frac{\beta^3}{\eps^2},
\]
where $c_{\rm PE}>0$ is a universal constant.
\end{restateproposition}

\begin{proof}
We construct two instances whose values at one fixed query differ by
more than $2\eps$, while their context distributions are hard to
distinguish.

Set $N=n_{\rm learn}+m_{\rm query}$.
On instances that differ only in $\mu$, the only informative oracle
responses are samples from $\mu$: since $P$ is known and identical
across instances, transition samples can be simulated using internal
randomness. At any fixed query, the algorithm's output is therefore a
function of at most $N$ i.i.d.\ context samples and independent internal
randomness. Unused samples may be added to obtain exactly $N$ draws.
This representation also covers adaptive sampling and early stopping.

For distributions on $\{1,2\}$, write $p=\mu(1)$ and let
\[
    \operatorname{kl}(p,q)
    :=p\log\frac pq+(1-p)\log\frac{1-p}{1-q}
\]
denote the Bernoulli relative entropy.
The one-source specialization of the change-of-measure inequality in
\citet[Lemma~1]{kaufmann2016complexity} implies the following:
if a test $\psi$ satisfies
$\PP_{p_0}(\psi=1)\le1/6$ and
$\PP_{p_1}(\psi=0)\le1/6$, then
\begin{equation}
    N\operatorname{kl}(p_1,p_0)
    \ge \operatorname{kl}(5/6,1/6)
    =\frac23\log 5
    =:\kappa.
    \label{eq:lb-known-testing}
\end{equation}
We will also use the elementary inequality
\begin{equation}
    \operatorname{kl}(p,q)
    \le \frac{(p-q)^2}{q(1-q)},
    \qquad p,q\in(0,1),
    \label{eq:lb-bernoulli-upper}
\end{equation}
which follows from $\log t\le t-1$.

Let $\Xcal=\{x_0,x_T\}$, $\Acal=\{a_0\}$, and
$\Zcal=\{1,2\}$. State $x_T$ is absorbing with reward zero, while
the reward at $x_0$ is one. From $x_0$, context $1$ causes a self-loop
and context $2$ causes a transition to $x_T$.
These rewards and transitions are known; only $p=\mu(1)$ is unknown.

There is a unique policy $\pi$. Its averaged value satisfies
$\bar V_p^\pi(x_0)=1+\gamma p\bar V_p^\pi(x_0)$, hence
\begin{equation}
    \bar V_p^\pi(x_0)=\frac{1}{1-\gamma p},
    \qquad
    v(p):=V_p^\pi(x_0,1)
    =1+\frac{\gamma}{1-\gamma p}.
    \label{eq:lb-geometric-value}
\end{equation}
The fixed query in the proof is $(x_0,1)$, so the quantity to estimate
is $v(p)$.

Choose
\[
    \rho:=\frac{1}{2\gamma\beta},
    \qquad
    p_0:=1-\rho,
    \qquad
    p_1:=1-\rho+\Delta,
    \qquad
    \Delta:=\frac{9\eps}{2\gamma^2\beta^2}.
\]
Both context distributions are valid: $0<\rho\le1/2$ and
\[
    \frac{\Delta}{\rho}
    =\frac{9\eps}{\gamma\beta}
    <\frac{1}{2\gamma}\le1.
\]
Since $1-\gamma p_0=3/(2\beta)$,
\eqref{eq:lb-geometric-value} gives
\begin{align}
    v(p_1)-v(p_0)
    &=\frac{\gamma^2\Delta}
    {\frac{3}{2\beta}
     \left(\frac{3}{2\beta}-\gamma\Delta\right)}
    \notag\\
    &>\frac{4\gamma^2\beta^2\Delta}{9}
    =2\eps.
    \label{eq:lb-pe-separation}
\end{align}
Thus an $\eps$-accurate estimate distinguishes the two instances.
On the other hand, \eqref{eq:lb-bernoulli-upper} yields
\begin{equation}
    \operatorname{kl}(p_1,p_0)
    \le\frac{\Delta^2}{\rho(1-\rho)}
    \le4\gamma\beta\Delta^2.
    \label{eq:lb-pe-information}
\end{equation}

Let $\widehat v$ be the algorithm's output at $(x_0,1)$ and define
\[
    \psi=\mathbf 1\!\left\{
    \widehat v\ge\frac{v(p_0)+v(p_1)}2
    \right\}.
\]
By \eqref{eq:lb-pe-separation}, the PE guarantee bounds each testing
error by $1/6$. Combining \eqref{eq:lb-known-testing} and
\eqref{eq:lb-pe-information},
\[
    N\ge\frac{\kappa}{4\gamma\beta\Delta^2}
    =\frac{\kappa\gamma^3\beta^3}{81\eps^2}
    \ge\frac{\kappa}{648}\frac{\beta^3}{\eps^2}.
\]
This proves the proposition with $c_{\rm PE}=\kappa/648$.
\end{proof}

\subsubsection{Best Value Estimation}
\label{subsec:lb-bve-known}

\begin{restateproposition}{prop:lb-main}[BVE]
Fix $\gamma\in[1/2,1)$, let $\beta=(1-\gamma)^{-1}$, and take
$\eps\in(0,\beta/18)$.
There exist two contextual MDPs with two states, one action, and two
contexts, sharing the same known transition kernel and reward function,
such that any algorithm satisfying the BVE objective of
Definition~\ref{def:pac-objectives} on both instances with confidence
$\delta=1/6$ must satisfy
\[
    n_{\rm learn}+m_{\rm query}
    \ge c_{\rm PE}\frac{\beta^3}{\eps^2},
\]
where $c_{\rm PE}>0$ is a universal constant.
\end{restateproposition}

\begin{proof}
The construction given appendix ~\ref{subsec:lb-pe-known} has a single
action. Therefore $V^\star=V^\pi$ for its unique policy $\pi$, and the
BVE objective coincides with PE. The lower bound follows directly
from that proposition.
\end{proof}

\subsubsection{Best Policy Extraction}
\label{subsec:lb-bpe-known}

\begin{restateproposition}{prop:lb-main}[BPE]
\label{prop:lb-bpe-known}
Fix $\gamma\in[1/2,1)$, let $\beta=(1-\gamma)^{-1}$, and take
$\eps\in(0,\beta/18)$.
There exist two contextual MDPs with four states, two actions, and
two contexts, sharing the same known transition kernel and reward
function, such that any algorithm satisfying the BPE objective of
Definition~\ref{def:pac-objectives} on both instances with confidence
$\delta=1/6$ must satisfy
\[
    n_{\rm learn}+m_{\rm query}
    \ge c_{\rm BPE}\frac{\beta^3}{\eps^2},
\]
where $c_{\rm BPE}>0$ is a universal constant.
\end{restateproposition}

\begin{proof}
We adapt the construction of appendix ~\ref{subsec:lb-pe-known} so
that identifying an $\eps$-optimal action distinguishes two context
distributions.

Let $\Xcal=\{x_0,x_1,x_2,x_T\}$,
$\Acal=\{a_1,a_2\}$, and $\Zcal=\{1,2\}$.
Rewards are zero at $x_0,x_T$ and one at $x_1,x_2$.
State $x_T$ is absorbing. At $x_0$, action $a_i$ leads
deterministically to $x_i$, independently of the context.
Actions have no effect after leaving $x_0$.

Set $\rho=1/(2\gamma\beta)$ as in
appendix~\ref{subsec:lb-pe-known}, and let
$q=1-\rho/(1-\rho)\in[0,1)$.
At each branch state, the following table gives the probability of
a self-loop; the remaining probability leads to $x_T$:
\[
\begin{array}{c|cc}
    & z=1 & z=2\\ \hline
    x_1 & 1 & 0\\
    x_2 & q & 1
\end{array}
\]
All these transitions are known. The two instances differ only in
their context distributions:
\[
    \mu_\sigma(1)=1-\rho+\sigma\Delta,
    \qquad
    \mu_\sigma(2)=\rho-\sigma\Delta,
    \qquad \sigma\in\{-1,+1\},
\]
where
\[
    \Delta:=\frac{3\eps}{\gamma^2\beta^2}.
\]
Our parameter range ensures
\[
    \frac{\Delta}{\rho}
    =\frac{6\eps}{\gamma\beta}
    <\frac{1}{3\gamma}\le\frac23,
\]
so both distributions have strictly positive coordinates.

The averaged self-loop probabilities at the branch states are
\[
    \bar p_{1,\sigma}=1-\rho+\sigma\Delta,
    \qquad
    \bar p_{2,\sigma}
    =1-\rho-\sigma\frac{\rho}{1-\rho}\Delta.
\]
Applying the absorbing-chain value formula
\eqref{eq:lb-geometric-value} to each branch gives
\[
    \bar V^\star_{\mu_\sigma}(x_i)
    =\frac{1}{1-\gamma\bar p_{i,\sigma}},
    \qquad i\in\{1,2\}.
\]
Because
$\bar p_{1,\sigma}-\bar p_{2,\sigma}
=\sigma\Delta/(1-\rho)$,
the optimal root action is $a_1$ under $\mu_{+1}$ and $a_2$ under
$\mu_{-1}$.

The loss from choosing the wrong root action is
\begin{align}
    g_\sigma
    &:=\gamma\left|
    \bar V^\star_{\mu_\sigma}(x_1)
    -\bar V^\star_{\mu_\sigma}(x_2)
    \right|\notag\\
    &=\frac{\gamma^2\Delta}
    {(1-\rho)
     (1-\gamma\bar p_{1,\sigma})
     (1-\gamma\bar p_{2,\sigma})}.
    \label{eq:lb-bpe-action-gap}
\end{align}
For either sign $\sigma$, one of the two denominator factors is at
most $3/(2\beta)$ and the other is at most
$3/(2\beta)+\gamma\Delta$.
Since $\Delta<2\rho/3$, we have $\gamma\Delta<1/(3\beta)$, and hence
\[
    (1-\gamma\bar p_{1,\sigma})
    (1-\gamma\bar p_{2,\sigma})
    <\frac{11}{4\beta^2}<\frac3{\beta^2}.
\]
Consequently,
\begin{equation}
    g_\sigma>\frac{\gamma^2\beta^2\Delta}{3}=\eps.
    \label{eq:lb-bpe-separation}
\end{equation}

Fix the query $(x_0,1)$ and let $\widehat\Pi$ denote the random
deterministic policy returned by the mechanism, conditional on its
internal randomization. Since actions have no effect after $x_0$,
its value at this query depends only on $\widehat\Pi(x_0,1)$.
By \eqref{eq:lb-bpe-separation}, selecting the wrong action incurs
loss greater than $\eps$. Thus the BPE guarantee implies that
\[
    \psi:=\mathbf1\{\widehat\Pi(x_0,1)=a_1\}
\]
has error probability at most $1/6$ under each instance.

The statistical reduction in appendix \ref{subsec:lb-pe-known} applies unchanged.
Using \eqref{eq:lb-bernoulli-upper},
\begin{align*}
    \operatorname{kl}
    \bigl(\mu_{+1}(1),\mu_{-1}(1)\bigr)
    &\le
    \frac{4\Delta^2}{(\rho+\Delta)(1-\rho-\Delta)}\\
    &\le\frac{24\Delta^2}{\rho}
    =48\gamma\beta\Delta^2,
\end{align*}
where we used $\rho+\Delta\ge\rho$ and
$1-\rho-\Delta>1-5\rho/3\ge1/6$.
Applying \eqref{eq:lb-known-testing} now yields
\[
    N\ge\frac{\kappa}{48\gamma\beta\Delta^2}
    =\frac{\kappa\gamma^3\beta^3}{432\eps^2}
    \ge\frac{\kappa}{3456}\frac{\beta^3}{\eps^2}.
\]
This proves the proposition with $c_{\rm BPE}=\kappa/3456$.
\end{proof}

\subsection{Upper Bounds}
\label{proof:mu-known-ub-star}

\begin{algorithm}[H]
\caption{\textsc{HalfErr}: variance-reduced error halving}
\label{alg:known-halferr}
\begin{algorithmic}[1]
\Require Initial vector $\bar v^{(0)}$; error radius $u\in(0,\beta]$;
failure probability $\eta\in(0,1)$; exact backup $g$;
$\beta$; sampling oracle for $\mu$.
\State $R\gets\lceil\beta\log4\rceil$,
$L\gets\log(16|\Xcal|R/\eta)$.
\State $n_1\gets
\left\lceil2^{16}(\beta^3/u^2+\beta^2)L\right\rceil$,
$n_2\gets\lceil512\beta^2L\rceil$,
$\alpha_1\gets L/n_1$.

\State Draw $Z_1,\ldots,Z_{n_1}\stackrel{\mathrm{i.i.d.}}{\sim}\mu$,
independently of all preceding calls.
\For{each $x\in\Xcal$}
    \State $\displaystyle
    \tilde w(x)\gets
    \frac1{n_1}\sum_{j=1}^{n_1}
    g_{x,\bar v^{(0)}}(Z_j)$.
    \State $\displaystyle
    \hat\sigma(x)\gets
    \frac1{n_1}\sum_{j=1}^{n_1}
    \bigl(g_{x,\bar v^{(0)}}(Z_j)\bigr)^2-\tilde w(x)^2$.
    \State $\displaystyle
    w(x)\gets\tilde w(x)
    -\sqrt{2\alpha_1\hat\sigma(x)}
    -4\alpha_1^{3/4}\beta-\frac23\alpha_1\beta$.
\EndFor

\For{$i=1,\ldots,R$}
    \State Draw a fresh batch
    $Z_1^{(i)},\ldots,Z_{n_2}^{(i)}
    \stackrel{\mathrm{i.i.d.}}{\sim}\mu$,
    independently of all preceding samples.
    \For{each $x\in\Xcal$}
        \State $\displaystyle
        \Delta^{(i)}(x)\gets
        \frac1{n_2}\sum_{j=1}^{n_2}
        \left[
        g_{x,\bar v^{(i-1)}}(Z_j^{(i)})
        -g_{x,\bar v^{(0)}}(Z_j^{(i)})
        \right]
        -\frac{u}{16\beta}$.
    \EndFor
    \State $\bar v^{(i)}
    \gets\max\{\bar v^{(i-1)},w+\Delta^{(i)}\}$,
    coordinatewise.
\EndFor
\State \Return $\bar v^{(R)}$.
\Statex Each context batch is shared across all controlled states;
both terms of each increment use the same contexts.
\end{algorithmic}
\end{algorithm}

\subsubsection{Best Value Estimation}

\begin{restatetheorem}{thm:main-mu-unknown}[BVE]
    \label{thm:main-ub-mu-unknown}
    Assume that $P$ and $r$ are known and that $\mu$ is accessible through
i.i.d.\ samples. Let $\gamma\in[0,1)$, $\beta=(1-\gamma)^{-1}$,
$\eps\in(0,1]$, and $\delta\in(0,1)$. Define
\[
    K:=\max\!\left\{1,
       \left\lceil\log_2\frac{\beta}{\eps}\right\rceil\right\},
    \qquad R:=\lceil\beta\log 4\rceil,
    \qquad L:=\log\frac{16|\Xcal|KR}{\delta}.
\]
There is an offline algorithm that uses at most
\begin{equation}
    n_{\rm learn}
    \le C\beta^3\bigl(\eps^{-2}+K\bigr)L
    =\widetilde O\!\left(\frac{\beta^3}{\eps^2}\right)
    \label{eq:known-sample-budget}
\end{equation}
samples from $\mu$ and returns $\hat{\bar V}^\star$ such that, with
probability at least $1-\delta$,
\begin{equation}
    0\le\hat{\bar V}^\star\le T\hat{\bar V}^\star,
    \qquad
    0\le\bar V^\star-\hat{\bar V}^\star\le\eps\mathbf 1.
    \label{eq:known-final-subsolution}
\end{equation}
On this same event, the query mechanism
$\widehat{\Vcal}^\star(x,z):=g_{x,\hat{\bar V}^\star}^\star(z)$
satisfies
\[
    \sup_{(x,z)\in\Xcal\times\Zcal}
    \bigl|\widehat{\Vcal}^\star(x,z)-V^\star(x,z)\bigr|
    \le\gamma\eps\le\eps.
\]
Thus BVE is achieved with $m_{\rm query}=0$. The constant $C$ is
universal, and the bound has no dependence on $|\Zcal|$.
\end{restatetheorem}

We use the Bellman expressions $g_{x,\bar v}^\star$, $g_{x,\bar v}^\pi$
and operators $T,T^\pi$ defined in the main. All vector inequalities
below are coordinatewise. The proof follows the variance reduction
and monotone iteration scheme of
\citet[Algorithms~1--2]{sidford2019nearoptimaltimesamplecomplexities}.
Here, each sampled context allows us to compute the Bellman expression
at every controlled state using the known model. Consequently, a
batch of $n$ contexts costs $n$ oracle calls, even when it is reused
across all states.

We first bound how contextual sampling errors accumulate, then prove
that one call to \textsc{HalfErr} halves the current error. Repeated
calls yield the theorem.

\paragraph{The contextual variance bound.}
For a deterministic context-aware policy $\pi$, define
\[
    \bar P^\pi(x'\mid x)
    :=\E_{Z\sim\mu}[P(x'\mid x,\pi(x,Z),Z)],
    \qquad
    v^\pi(x)
    :=\operatorname{Var}_{Z\sim\mu}
       \bigl(g_{x,\bar V^\pi}^\pi(Z)\bigr).
\]

\begin{lemma}[Accumulated contextual variance]
\label{lem:variance-bound}
For every deterministic context-aware policy $\pi$,
\[
    (I-\gamma^2\bar P^\pi)^{-1}v^\pi
    =\sum_{t\ge0}\gamma^{2t}(\bar P^\pi)^t v^\pi
    \le\beta^2\mathbf 1.
\]
\end{lemma}

\begin{proof}
Start at $X_0=x$, with $Z_0\sim\mu$, and let
$G=\sum_{t\ge0}\gamma^t r(X_t,\pi(X_t,Z_t),Z_t)$.
Write $\mathcal F_t$ for the history immediately before observing
$Z_t$, including $X_t$, and set
$\mathcal G_t=\sigma(\mathcal F_t,Z_t)$.
The martingale $M_t=\E[G\mid\mathcal F_t]$ converges to $G$ in $L^2$,
since $G\in[0,\beta]$.

Split each increment into the context and transition contributions:
\[
    M_{t+1}-M_t=D_t^Z+D_t^X,
    \qquad
    D_t^Z:=\E[G\mid\mathcal G_t]-M_t,
    \qquad
    D_t^X:=M_{t+1}-\E[G\mid\mathcal G_t].
\]
These contributions are orthogonal. By independence of the fresh
context and the Bellman identity,
\[
    D_t^Z=\gamma^t
    \bigl(g_{X_t,\bar V^\pi}^\pi(Z_t)-\bar V^\pi(X_t)\bigr),
    \qquad
    \E[(D_t^Z)^2\mid\mathcal F_t]
    =\gamma^{2t}v^\pi(X_t).
\]
The controlled process $(X_t)$ has transition matrix $\bar P^\pi$.
Orthogonality of the martingale increments therefore gives
\[
    \sum_{t\ge0}\gamma^{2t}(\bar P^\pi)^t v^\pi(x)
    \le\operatorname{Var}(G\mid X_0=x)
    \le\beta^2.
\]
This holds for every $x$.
\end{proof}

\begin{corollary}[Resolvent variance bound]
\label{cor:V-star-bound}
For every deterministic context-aware policy $\pi$,
\[
    \bigl\|(I-\gamma\bar P^\pi)^{-1}\sqrt{v^\pi}\bigr\|_\infty
    \le\sqrt{2}\,\beta^{3/2}.
\]
\end{corollary}

\begin{proof}
For $\gamma>0$, apply
\citet[Lemma~C.3]{sidford2019nearoptimaltimesamplecomplexities}:
\[
    \bigl\|(I-\gamma\bar P^\pi)^{-1}\sqrt{v^\pi}\bigr\|_\infty
    \le
    \sqrt{\frac{1+\gamma}{1-\gamma}}
    \bigl\|(I-\gamma^2\bar P^\pi)^{-1}v^\pi\bigr\|_\infty^{1/2}.
\]
Lemma~\ref{lem:variance-bound} proves the claim. The case $\gamma=0$
follows directly from the same lemma.
\end{proof}

\paragraph{One call to \textsc{HalfErr}.}
\label{sec:halferr-algorithm}
The input is a vector $\bar v^{(0)}$ satisfying
\begin{equation}
    0\le\bar v^{(0)}\le T\bar v^{(0)},
    \qquad
    0\le\bar V^\star-\bar v^{(0)}\le u\mathbf 1,
    \qquad u\in(0,\beta].
    \label{eq:known-halferr-input}
\end{equation}
We estimate $T\bar v^{(0)}$ once, then estimate only the change in
the Bellman backup at subsequent iterates. This change has magnitude
at most $\gamma u$, which allows smaller batches in the inner loop.

For a failure probability $\eta\in(0,1)$, set
\begin{align}
    R&:=\lceil\beta\log4\rceil,
    &L&:=\log\frac{16|\Xcal|R}{\eta},\notag\\
    n_1&:=\left\lceil c_2(\beta^3/u^2+\beta^2)L\right\rceil,
    &n_2&:=\lceil c_3\beta^2L\rceil,
    \qquad \alpha_1:=L/n_1.
    \label{eq:known-halferr-parameters}
\end{align}
The routine consists of the following two steps.
\begin{enumerate}[leftmargin=*]
\item \emph{Estimate the anchor.}
Draw $Z_1,\ldots,Z_{n_1}\stackrel{\mathrm{i.i.d.}}{\sim}\mu$.
For every $x\in\Xcal$, compute
\begin{align*}
    \tilde w(x)
    &:=\frac1{n_1}\sum_{j=1}^{n_1}g_{x,\bar v^{(0)}}^\star(Z_j),\\
    \hat\sigma(x)
    &:=\frac1{n_1}\sum_{j=1}^{n_1}
        (g_{x,\bar v^{(0)}}^\star(Z_j))^2-\tilde w(x)^2,\\
    w(x)
    &:=\tilde w(x)-\sqrt{2\alpha_1\hat\sigma(x)}
        -4\alpha_1^{3/4}\beta-\frac23\alpha_1\beta.
\end{align*}

\item \emph{Estimate small increments.}
For $i=1,\ldots,R$, draw a fresh i.i.d.\ batch
$Z_1^{(i)},\ldots,Z_{n_2}^{(i)}\sim\mu$ and compute
\begin{align*}
    \Delta^{(i)}(x)
    &:=\frac1{n_2}\sum_{j=1}^{n_2}
       \left[g_{x,\bar v^{(i-1)}}^\star(Z_j^{(i)})
            -g_{x,\bar v^{(0)}}^\star(Z_j^{(i)})\right]
       -\frac{u}{16\beta},\\
    \bar v^{(i)}
    &:=\max\{\bar v^{(i-1)},w+\Delta^{(i)}\},
\end{align*}
where the maximum is coordinatewise. Return $\bar v^{(R)}$.
\end{enumerate}

Take $c_2\ge2^{16}$ and $c_3\ge512$.
All batches are independent across iterations and calls, and the
same batch is reused across states. The additional $\beta^2$ term in
$n_1$ controls the early calls, for which $u$ can be as large as $\beta$.

For the analysis, define
\[
    v_{\bar v}(x)
    :=\operatorname{Var}_{Z\sim\mu}
       \bigl(g_{x,\bar v}^\star(Z)\bigr),
    \qquad v_0:=v_{\bar v^{(0)}},
    \qquad v_\star:=v_{\bar V^\star}.
\]
Whenever $0\le\bar v\le\beta\mathbf 1$, rewards in $[0,1]$ give
$0\le g_{x,\bar v}^\star(z)\le\beta$. Moreover,
\begin{equation}
    |g_{x,\bar v}^\star(z)-g_{x,\bar w}^\star(z)|
    \le\gamma\|\bar v-\bar w\|_\infty.
    \label{eq:known-backup-lipschitz}
\end{equation}

\begin{lemma}[Stability of the standard deviation]
\label{lem:sqrt-stab}
If $\|\bar v-\bar V^\star\|_\infty\le u$, then
\[
    |\sqrt{v_{\bar v}}-\sqrt{v_\star}|
    \le\gamma u\mathbf 1\le u\mathbf 1.
\]
\end{lemma}

\begin{proof}
The reverse triangle inequality for centered $L^2(\mu)$ functions
and \eqref{eq:known-backup-lipschitz} give, for every $x$,
\[
    |\sqrt{v_{\bar v}(x)}-\sqrt{v_\star(x)}|
    \le
    \sqrt{\operatorname{Var}_\mu
      (g_{x,\bar v}^\star(Z)-g_{x,\bar V^\star}^\star(Z))}
    \le\gamma u.
\]
\end{proof}

\paragraph{Concentration of the two estimates.}
The next two lemmas apply conditionally on all samples preceding the
current call. Thus a random input $\bar v^{(0)}$ is allowed, provided
it satisfies \eqref{eq:known-halferr-input}.

\begin{lemma}[Anchor concentration]
\label{lem:anchor-error}
With probability at least $1-\eta/2$, simultaneously for all $x$,
\begin{align}
    |\tilde w(x)-(T\bar v^{(0)})(x)|
    &\le\sqrt{2\alpha_1v_0(x)}+\frac23\alpha_1\beta,
    \label{eq:anchor-mean}\\
    |\hat\sigma(x)-v_0(x)|
    &\le4\beta^2\sqrt{2\alpha_1}.
    \label{eq:anchor-var}
\end{align}
\end{lemma}

\begin{proof}
Apply the scalar concentration bounds in
\citet[Lemma~5.1]{sidford2019nearoptimaltimesamplecomplexities} to
$g_{x,\bar v^{(0)}}^\star(Z)$, which takes values in $[0,\beta]$.
A union bound over the controlled states gives the result because
$L\ge\log(16|\Xcal|/\eta)$. Independence between different states is
not required.
\end{proof}

\begin{lemma}[Increment concentration]
\label{lem:inner-error}
Condition on the samples preceding inner iteration $i$.
If $\|\bar v^{(i-1)}-\bar v^{(0)}\|_\infty\le u$, then, with
conditional probability at least $1-\eta/(2R)$,
\[
    T\bar v^{(i-1)}-T\bar v^{(0)}-\frac{u}{8\beta}\mathbf1
    \le\Delta^{(i)}
    \le T\bar v^{(i-1)}-T\bar v^{(0)}.
\]
\end{lemma}

\begin{proof}
By \eqref{eq:known-backup-lipschitz}, each sampled difference belongs
to $[-\gamma u,\gamma u]$. Apply Hoeffding's inequality in the form
of \citet[Theorem~E.1]{sidford2019nearoptimaltimesamplecomplexities}
conditionally on the past, then take a union bound over $x$.
The empirical increment differs from its mean by at most
\[
    \gamma u\sqrt{\frac{2\log(4|\Xcal|R/\eta)}{n_2}}
    \le\frac{u}{16\beta}.
\]
Subtracting $u/(16\beta)$ proves the claim.
\end{proof}

\begin{lemma}[One-step approximation]
\label{lem:sandwich}
On the events of Lemmas~\ref{lem:anchor-error} and
\ref{lem:inner-error},
\[
    T\bar v^{(i-1)}-\xi
    \le w+\Delta^{(i)}\le T\bar v^{(i-1)},
\]
where
\begin{equation}
    \xi:=2\sqrt{2\alpha_1v_0}
    +\left(\frac{u}{8\beta}
           +8\alpha_1^{3/4}\beta
           +\frac43\alpha_1\beta\right)\mathbf1.
    \label{eq:xi-def}
\end{equation}
\end{lemma}

\begin{proof}
Equation~\eqref{eq:anchor-var} implies
$|\sqrt{\hat\sigma(x)}-\sqrt{v_0(x)}|
\le2\beta(2\alpha_1)^{1/4}$.
Substitution in the definition of $w$, together with
\eqref{eq:anchor-mean}, gives
\[
    T\bar v^{(0)}-2\sqrt{2\alpha_1v_0}
    -\left(8\alpha_1^{3/4}\beta+\frac43\alpha_1\beta\right)\mathbf1
    \le w\le T\bar v^{(0)}.
\]
Adding the bounds in Lemma~\ref{lem:inner-error} proves the result.
\end{proof}

\paragraph{Monotonicity and error propagation.}
Let $\pi^\star$ be a deterministic optimal context-aware policy.
Then $T^{\pi^\star}\bar V^\star=\bar V^\star$ and
$v_\star=v^{\pi^\star}$.

\begin{lemma}[Invariants and contraction]
\label{lem:contraction}
With probability at least $1-\eta$, all inner iterations satisfy:
\begin{enumerate}[label=\rm(\roman*),leftmargin=*]
    \item $0\le\bar v^{(0)}\le\bar v^{(i-1)}
           \le\bar v^{(i)}\le\bar V^\star$;
    \item $\bar v^{(i)}\le T\bar v^{(i)}$;
    \item $\|\bar v^{(i)}-\bar v^{(0)}\|_\infty\le u$;
    \item $\bar V^\star-\bar v^{(i)}
          \le\gamma\bar P^{\pi^\star}
             (\bar V^\star-\bar v^{(i-1)})+\xi$.
\end{enumerate}
\end{lemma}

\begin{proof}
Suppose the anchor estimate and all preceding increments satisfy
their concentration bounds. Inductively, the previous iterate is a
subsolution and is at most $\bar V^\star$. Lemma~\ref{lem:sandwich}
and the update then imply
\[
    \bar v^{(i-1)}\le\bar v^{(i)}
    \le T\bar v^{(i-1)}\le\bar V^\star,
    \qquad
    \bar v^{(i)}\le T\bar v^{(i-1)}\le T\bar v^{(i)}.
\]
This proves (i)--(ii), and (iii) follows from
$\bar V^\star-\bar v^{(0)}\le u\mathbf1$.
The increment lemma is therefore applicable at every step on the
preceding success event. Its conditional failure probabilities,
together with the anchor failure probability, sum to at most
$\eta/2+R\eta/(2R)=\eta$.

Finally, the lower bound in Lemma~\ref{lem:sandwich} gives
$\bar v^{(i)}\ge T\bar v^{(i-1)}-\xi$.
Since $T\bar v^{(i-1)}\ge T^{\pi^\star}\bar v^{(i-1)}$,
\[
    \bar V^\star-T\bar v^{(i-1)}
    \le T^{\pi^\star}\bar V^\star
       -T^{\pi^\star}\bar v^{(i-1)}
    =\gamma\bar P^{\pi^\star}
       (\bar V^\star-\bar v^{(i-1)}),
\]
which proves (iv).
\end{proof}

\begin{proposition}[Halving guarantee]
\label{prop:halving}
Under \eqref{eq:known-halferr-input}, the halving routine returns,
with probability at least $1-\eta$, a vector satisfying
\[
    0\le\bar v^{(R)}\le T\bar v^{(R)},
    \qquad
    0\le\bar V^\star-\bar v^{(R)}\le\frac u2\mathbf1.
\]
It uses at most
$C\beta^3(1+u^{-2})\log(16|\Xcal|R/\eta)$ samples from $\mu$.
\end{proposition}

\begin{proof}
Work on the event of Lemma~\ref{lem:contraction} and set
$H:=(I-\gamma\bar P^{\pi^\star})^{-1}$.
Iterating its recursion gives
\[
    0\le\bar V^\star-\bar v^{(R)}
    \le\gamma^R u\mathbf1+H\xi.
\]
The choice $R=\lceil\beta\log4\rceil$ ensures $\gamma^R\le1/4$.
It remains to bound $H\xi$.

By Lemma~\ref{lem:sqrt-stab},
$\sqrt{v_0}\le\sqrt{v_\star}+u\mathbf1$.
Using Corollary~\ref{cor:V-star-bound} for $\pi^\star$ and
$H\mathbf1=\beta\mathbf1$, we obtain
\begin{equation}
\begin{split}
    \|H\xi\|_\infty
    \le{}&\frac u8
       +4\sqrt{\alpha_1}\,\beta^{3/2}
       +2\sqrt{2\alpha_1}\,u\beta\\
       &+8\alpha_1^{3/4}\beta^2
       +\frac43\alpha_1\beta^2.
\end{split}
\label{eq:known-propagated-error}
\end{equation}
The corrected anchor budget gives both bounds
\[
    \alpha_1\le\frac{u^2}{c_2\beta^3},
    \qquad
    \alpha_1\le\frac1{c_2\beta^2}.
\]
Consequently,
\begin{align*}
    \sqrt{\alpha_1}\,\beta^{3/2}
    &\le\frac{u}{\sqrt{c_2}},
    &\sqrt{\alpha_1}\,u\beta
    &\le\frac{u}{\sqrt{c_2}},\\
    \alpha_1^{3/4}\beta^2
    &\le\frac{u}{c_2^{3/4}},
    &\alpha_1\beta^2
    &\le\frac{u}{c_2}.
\end{align*}
For the third inequality, multiply the square root of the first
bound on $\alpha_1$ by the fourth root of the second; the last uses
$u\le\beta$.
Substituting in \eqref{eq:known-propagated-error} yields
\[
    \|H\xi\|_\infty
    \le u\left(\frac18+\frac{4+2\sqrt2}{\sqrt{c_2}}
                    +\frac8{c_2^{3/4}}+\frac4{3c_2}\right)
    \le\frac u4.
\]
This proves the error bound, and the subsolution property follows
from Lemma~\ref{lem:contraction}(ii).
Finally, $R=O(\beta)$ gives
\[
    n_1+Rn_2
    \le C L\left(\frac{\beta^3}{u^2}+\beta^2+\beta^3\right)
    \le C\beta^3(1+u^{-2})L.
\]
\end{proof}

\paragraph{From halving to the BVE guarantee.}
\label{sec:main-ub}
\begin{proof}[Proof of Theorem~\ref{thm:main-ub-mu-unknown}]
Initialize $\bar v_0=\mathbf0$ and $u_0=\beta$.
For $k=1,\ldots,K$, run
\[
    \bar v_k:=\textsc{HalfErr}(\bar v_{k-1},u_{k-1},\delta/K),
    \qquad u_k:=u_{k-1}/2,
\]
using fresh samples, and return $\hat{\bar V}^\star:=\bar v_K$.
The initialization satisfies \eqref{eq:known-halferr-input}.
Conditional on all previous calls succeeding,
Proposition~\ref{prop:halving} applies to the next call.
Therefore all calls succeed with probability at least $1-\delta$.
On that event, their subsolution guarantees and
$u_K=\beta2^{-K}\le\eps$ give
\eqref{eq:known-final-subsolution}.

Every call uses the same logarithmic factor
$L=\log(16|\Xcal|KR/\delta)$, and
\[
    \sum_{k=1}^K u_{k-1}^{-2}
    =\frac{4^K-1}{3\beta^2}
    \le\frac4{3\eps^2}.
\]
Summing the sample bound in Proposition~\ref{prop:halving} proves
\eqref{eq:known-sample-budget}.

Finally, $V^\star(x,z)=g_{x,\bar V^\star}^\star(z)$, so
\eqref{eq:known-backup-lipschitz} gives, simultaneously for every
$(x,z)$,
\[
    |\widehat{\Vcal}^\star(x,z)-V^\star(x,z)|
    \le\gamma\|\hat{\bar V}^\star-\bar V^\star\|_\infty
    \le\gamma\eps.
\]
Evaluating this expression uses only $P,r$, and the stored vector;
it requires no additional samples.
\end{proof}

For comparison with the main theorem, \eqref{eq:known-sample-budget}
also implies the simpler explicit bound
\begin{equation}
    n_{\rm learn}
    =O\!\left(
    \frac{\beta^3}{\eps^2}
    \log\frac{8|\Xcal|\beta}{\delta\eps}
    \log\frac{4\beta}{\eps}\right).
    \label{eq:known-budget-main}
\end{equation}
No union bound over contexts is used: a single event controlling the
averaged vector implies the guarantee at every state-context pair.

\subsubsection{Policy Evaluation}
\begin{restatetheorem}{thm:main-mu-unknown}[PE]
    \label{thm:pe-ub}
Under the assumptions of Theorem~\ref{thm:main-ub-mu-unknown}, let
$\pi:\Xcal\times\Zcal\to\Acal$ be a fixed policy supplied as input.
There is an algorithm using the offline budget
\eqref{eq:known-sample-budget} that returns $\hat{\bar V}^\pi$ such
that, with probability at least $1-\delta$,
\[
    \|\hat{\bar V}^\pi-\bar V^\pi\|_\infty\le\eps,
    \qquad
    \sup_{(x,z)\in\Xcal\times\Zcal}
    |\widehat{\Vcal}^\pi(x,z)-V^\pi(x,z)|\le\gamma\eps\le\eps,
\]
where
$\widehat{\Vcal}^\pi(x,z):=g_{x,\hat{\bar V}^\pi}^\pi(z)$.
Thus PE is achieved with $m_{\rm query}=0$.
\end{restatetheorem}

\begin{proof}
Run the construction of Theorem~\ref{thm:main-ub-mu-unknown}, replacing
$g^\star,T,\bar V^\star$ by $g^\pi,T^\pi,\bar V^\pi$.
The range and Lipschitz bounds are unchanged, so
Lemmas~\ref{lem:sqrt-stab}--\ref{lem:sandwich} apply with these
replacements. The monotone update preserves a subsolution of $T^\pi$.
The only operator comparison needed for
Lemma~\ref{lem:contraction}(iv) becomes the identity
\[
    \bar V^\pi-T^\pi\bar v
    =\gamma\bar P^\pi(\bar V^\pi-\bar v).
\]
Corollary~\ref{cor:V-star-bound} already holds for this policy.
Hence Proposition~\ref{prop:halving} and the outer-loop proof apply
with the same budgets and failure probabilities.

The pointwise guarantee follows from the main's Bellman identity
$V^\pi(x,z)=g_{x,\bar V^\pi}^\pi(z)$:
\[
    |\widehat{\Vcal}^\pi(x,z)-V^\pi(x,z)|
    \le\gamma\|\hat{\bar V}^\pi-\bar V^\pi\|_\infty
    \le\gamma\eps.
\]
The query again uses no additional samples.
\end{proof}

\subsubsection{Best Policy Extraction}
\begin{restatetheorem}{thm:main-mu-unknown}[BPE]
    \label{thm:bpe-ub}
Under the assumptions of Theorem~\ref{thm:main-ub-mu-unknown}, let
$\hat{\bar V}^\star$ be its offline output and define, with fixed
deterministic tie-breaking,
\[
    \widehat\Pi(x,z)\in\arg\max_{a\in\Acal}
    \left\{r(x,a,z)+\gamma\sum_{x'}
        P(x'\mid x,a,z)\hat{\bar V}^\star(x')\right\}.
\]
With probability at least $1-\delta$,
\[
    \sup_{(x,z)\in\Xcal\times\Zcal}
    \bigl(V^\star(x,z)-V^{\widehat\Pi}(x,z)\bigr)
    \le\gamma\eps\le\eps.
\]
The offline budget is \eqref{eq:known-sample-budget} and
$m_{\rm query}=0$.
\end{restatetheorem}

\begin{proof}
Work on the success event of
Theorem~\ref{thm:main-ub-mu-unknown}.
The subsolution property and greediness imply
\[
    \hat{\bar V}^\star\le T\hat{\bar V}^\star
    =T^{\widehat\Pi}\hat{\bar V}^\star.
\]
Iterating the monotone contraction $T^{\widehat\Pi}$ gives
$\hat{\bar V}^\star\le\bar V^{\widehat\Pi}$.
Therefore, at every $(x,z)$,
\[
    V^{\widehat\Pi}(x,z)
    =g_{x,\bar V^{\widehat\Pi}}^{\widehat\Pi}(z)
    \ge g_{x,\hat{\bar V}^\star}^{\widehat\Pi}(z)
    =g_{x,\hat{\bar V}^\star}^\star(z)
    =\widehat{\Vcal}^\star(x,z).
\]
The BVE guarantee from
Theorem~\ref{thm:main-ub-mu-unknown} now yields
\[
    0\le V^\star(x,z)-V^{\widehat\Pi}(x,z)
    \le V^\star(x,z)-\widehat{\Vcal}^\star(x,z)
    \le\gamma\eps.
\]
The policy is computed from the same stored vector and known model,
so its extraction requires no additional samples.
\end{proof}
\section{Proof of Proposition~\ref{prop:exo-look}}
\label{proof:exo-look}

\begin{proof}
Let $\mathcal M_0=(\Xcal,\Acal,P_0,r_0,\gamma)$ denote the
underlying tabular MDP. We show that its perfect one-step
transition-lookahead model and the contextual MDP in the proposition
describe the same decision problem, with the same sampling access.

\paragraph{(i) Contextual-MDP structure.}
Define
\[
\Zcal:=\Xcal^{\Xcal\times\Acal},
\qquad
\mu:=\bigotimes_{(x,a)\in\Xcal\times\Acal}
P_0(\cdot\mid x,a).
\]
Thus a context $z\in\Zcal$ assigns a hypothetical successor $z(x,a)$
to every state-action pair, and a draw $Z\sim\mu$ has independent
coordinates with
$Z(x,a)\sim P_0(\cdot\mid x,a)$.

Set
\[
P(x'\mid x,a,z)=\mathbf 1\{x'=z(x,a)\},
\qquad
r(x,a,z)=r_0(x,a).
\]
In the lookahead model, after observing $(X_t,Z_t)$ and selecting
$A_t$, the next controlled state is
$X_{t+1}=Z_t(X_t,A_t)$, while $Z_{t+1}\sim\mu$ is drawn independently
of the past. Consequently, its augmented transition kernel is
\[
\Pr\!\left(
X_{t+1}=x',Z_{t+1}=z'
\,\middle|\,
X_t=x,Z_t=z,A_t=a
\right)
=
P(x'\mid x,a,z)\mu(z').
\]
Its reward is $r_0(x,a)=r(x,a,z)$. These are exactly the dynamics,
observation structure, and rewards of the contextual MDP defined
in Section~\ref{sec:setting}.

\paragraph{(ii) Equivalence of policies and values.}
Both models reveal the same pair $(x,z)$ before the action is chosen.
Hence every stationary context-aware policy in one model defines
the same policy in the other.

Fix such a policy $\pi$ and an initial pair $(x,z)$. Couple the two
models using the same future context draws and, if needed, the same
policy randomization. Their controlled states, actions, and rewards
then agree at every time step. Writing $V_{\mathrm{LA}}^\pi$ for the
value in the lookahead model, we obtain
\[
V_{\mathrm{LA}}^\pi(x,z)=V^\pi(x,z)
\qquad
\text{for every }(x,z)\in\Xcal\times\Zcal.
\]
Taking the supremum over the common class of deterministic stationary
policies gives
\[
V_{\mathrm{LA}}^\star(x,z)=V^\star(x,z).
\]

In particular, the averaged value
\[
\bar V^\star(x)
=
\E_{Z\sim\mu}[V_{\mathrm{LA}}^\star(x,Z)]
\]
is the optimal lookahead value before the current context is revealed.
It satisfies
\[
\bar V^\star(x)
=
\E_{Z\sim\mu}\!\left[
\max_{a\in\Acal}
\left\{
r_0(x,a)+\gamma\bar V^\star(Z(x,a))
\right\}
\right].
\]
The operator on the right-hand side is a $\gamma$-contraction in
the supremum norm, so this equation uniquely characterizes
$\bar V^\star$. The maximum is taken after observing $Z$, consistently
with the information available in the lookahead model.

The policywise value identity also implies that PE and BVE errors,
as well as the BPE loss
$V^\star(x,z)-V^{\widehat\Pi}(x,z)$, are identical in the two models.
Thus their decision-time PAC guarantees transfer without changing
$\varepsilon$ or $\delta$.

\paragraph{(iii) Information regime and sample accounting.}
A lookahead-oracle call returns the full tuple
$\bigl(Z(x,a)\bigr)_{(x,a)\in\Xcal\times\Acal}$.
Independent calls therefore provide i.i.d.\ samples from $\mu$.

Although $P_0$ is unknown, the contextual transition kernel $P$
has a known specification that does not depend on $P_0$:
for every supplied context $z$ and continuation vector $\bar v$,
\[
\sum_{x'\in\Xcal}P(x'\mid x,a,z)\bar v(x')
=
\bar v(z(x,a)).
\]
This expression requires only a table lookup. Likewise,
$r(x,a,z)=r_0(x,a)$ is known by assumption.
The unknown kernel $P_0$ enters the contextual model only through
the distribution $\mu$ of the observed transition tables.

The model therefore satisfies the assumptions of
Theorem~\ref{thm:main-mu-unknown}. Applying that theorem and using
the value identities above yields, for each of PE, BVE, and BPE,
\[
(n_{\rm learn},m_{\rm query})
=
\left(
\widetilde O\!\left(\frac{\beta^3}{\varepsilon^2}\right),
0
\right),
\qquad
\beta=(1-\gamma)^{-1},
\]
with the corresponding decision-time PAC guarantee on the
lookahead model.

Here the offline budget counts full transition-table samples.
If each table is generated through individual calls to a tabular
generative model for $P_0$, it costs $|\Xcal||\Acal|$ transition
samples. The corresponding total is therefore
$\widetilde O(|\Xcal||\Acal|\beta^3/\varepsilon^2)$ individual
transitions. Once the query context is observed, no additional
sampling is required.
\end{proof}

\section{Proofs : \texorpdfstring{$\mu$}{μ} and \texorpdfstring{$P$}{P} unknown}
\subsection{Lower Bounds}
\subsubsection{Policy Evaluation}
\label{sec:lower-bound-p-unkn}
\begin{proposition}[PE offline/query lower bounds]
\label{prop:lb-pe-unknown}
Fix $\gamma\in[1/2,1)$, $\beta=(1-\gamma)^{-1}$, $S\ge4$,
$\bar n\in\mathbb N$, and $\eps\in(0,\beta/24)$.
There exists a family of contextual MDPs with $S$ states,
one action, known rewards, a common known context distribution,
and $|\Zcal|\le C(1+\bar n\eps^2/\beta^2)$, such that any
mechanism satisfying PE on this family with $\delta\le1/12$
and $n_{\rm learn}\le\bar n$ must satisfy
\[
    m_{\rm query}\ge c_1\frac{\beta^2}{\eps^2},
    \qquad
    n_{\rm learn}+S\,m_{\rm query}
    \ge c_2\frac{S\beta^3}{\eps^2},
\]
where $c_1,c_2,C>0$ are universal constants.
\end{proposition}
\begin{proof}
Fix $\gamma\in[1/2,1)$, $\beta=(1-\gamma)^{-1}$, $S\ge4$,
$\bar n\in\mathbb N$, and $\eps\in(0,\beta/24)$.
Thus we may take $c_0=1/24$.
We construct a single family that forces both inequalities. One part
contains $d:=S-3$ branch states, each encoding a separate transition
parameter. A second part contains one query state with many possible
query contexts. Both arguments compare alternatives to the same
reference instance, allowing us to sum expected offline sample counts
even when sampling is adaptive.

\paragraph{The common construction.}
Let
\[
    \Xcal=\{x^{(1)},\ldots,x^{(d)},x_T,x_\eta,x_\beta\},
    \qquad \Acal=\{a_0\},
    \qquad \pi\equiv a_0.
\]
States $x_T$ and $x_\beta$ are absorbing, with rewards zero and one,
respectively. Each branch state $x^{(k)}$ has reward one, and the
query state $x_\eta$ has reward zero. Rewards do not depend on the
context and are identical across all instances.

For a universal constant $c_Z$ specified below, set
\[
    M:=\left\lceil c_Z\frac{\bar n\eps^2}{\beta^2}\right\rceil+1,
    \qquad
    \Zcal_{\bar n}:=\{z_0,z^{(1)},\ldots,z^{(M)}\},
\]
and use the same context distribution throughout the family:
\[
    \mu(z_0)=\frac12,
    \qquad \mu(z^{(j)})=\frac1{2M},\quad j\in\{1,\ldots,M\}.
\]
Only the transition kernel varies, so samples from $\mu$ carry no
information about the instance.

The family is indexed by
$(\boldsymbol\theta,\eta,i)\in\{0,1\}^d\times\{0,1\}\times\{1,\ldots,M\}$.
For the branch states, define
\[
    \bar p:=\frac{4\gamma-1}{3\gamma},
    \qquad p:=2\bar p-1=\frac{5\gamma-2}{3\gamma},
    \qquad \alpha:=\frac{8\eps}{\gamma\beta^2}.
\]
At $(x^{(k)},a_0,z_0)$, the next state is $x^{(k)}$ with probability
$p+\theta_k\alpha$ and $x_T$ otherwise. At every other context, the
branch state self-loops deterministically. Therefore its averaged
self-loop probability is
\begin{equation}
    \bar p_k(\theta_k)
    =\frac{p+\theta_k\alpha+1}{2}
    =\bar p+\frac{\theta_k\alpha}{2}.
    \label{eq:bar-pk}
\end{equation}
These transitions are valid because $p\ge1/3$, $1-p=2/(3\gamma\beta)$,
and
\[
    \frac{\alpha}{1-p}=\frac{12\eps}{\beta}<\frac12.
\]

At the query state, context $z_0$ leads deterministically to $x_T$.
For $j\in\{1,\ldots,M\}$, let
\begin{equation}
    P(x_\beta\mid x_\eta,a_0,z^{(j)})
    =\frac12+\eta\Delta\mathbf1\{j=i\},
    \qquad \Delta:=\frac{3\eps}{\gamma\beta},
    \label{eq:tradeoff-query-kernel}
\end{equation}
with the remaining probability assigned to $x_T$.
Since $\eps<\beta/24$ and $\gamma\ge1/2$, we have $0<\Delta<1/4$.
When $\eta=0$, every query context is unbiased, so the instance is
independent of $i$.

Write $\mathcal M_0$ for the reference instance
$(\boldsymbol\theta,\eta,i)=(\mathbf0,0,1)$, and $\mathbb E_0$ for
expectation under this instance.

\paragraph{The testing inequality used in both arguments.}
We reuse the midpoint-test reduction from
\eqref{eq:lb-known-testing} and the Bernoulli divergence bound
\eqref{eq:lb-bernoulli-upper}.
Suppose an alternative differs from $\mathcal M_0$ only at one
transition triplet $t$, whose Bernoulli parameters are $q_0$ and
$q_1$. If their values at a fixed query $y$ differ by more than
$2\eps$, the PE guarantee makes each midpoint-test error at most
$\delta\le1/12$.
The adaptive change-of-measure inequality of
\citet[Lemma~1]{kaufmann2016complexity} then gives
\begin{equation}
    \mathbb E_0\!\left[
        N_t^{\rm off}+N_t^{\rm query}(y)\right]
    \operatorname{kl}(q_0,q_1)
    \ge\kappa,
    \qquad \kappa:=\frac23\log5,
    \label{eq:tradeoff-testing}
\end{equation}
where the two counts record calls at $t$ during the offline phase
and when answering $y$.
All other oracle responses have identical laws in the two instances.
The lemma allows adaptive sampling and stopping; the budgets ensure
that the experiment is finite.

\paragraph{The offline/query tradeoff.}
We first prove
\begin{equation}
    n_{\rm learn}+d\,m_{\rm query}
    \ge\frac{\kappa d}{1152}\frac{\beta^3}{\eps^2}.
    \label{eq:offline-main-bound}
\end{equation}
For each $k$, compare $\mathcal M_0$ with the alternative in which
only $\theta_k$ is changed to one, and use the fixed query
$y_k=(x^{(k)},z_0)$.

By the absorbing-chain formula \eqref{eq:lb-geometric-value},
\[
    \bar V_b^\pi(x^{(k)})
    =\frac1{1-\gamma(\bar p+b\alpha/2)},
    \qquad b\in\{0,1\}.
\]
The main's pointwise Bellman identity consequently gives
\[
    f_b:=V_b^\pi(x^{(k)},z_0)
    =1+\frac{\gamma(p+b\alpha)}{\lambda-\gamma b\alpha/2},
    \qquad \lambda:=1-\gamma\bar p=\frac4{3\beta}.
\]
Since $\lambda+\gamma p/2=1-\gamma/2\ge1/2$,
\begin{align*}
    f_1-f_0
    &=\frac{\gamma\alpha(1-\gamma/2)}
            {\lambda(\lambda-\gamma\alpha/2)}\\
    &\ge\frac{\gamma\alpha}{2\lambda^2}
     =\frac{9\gamma\alpha\beta^2}{32}
     =\frac94\eps>2\eps.
\end{align*}

The only informative triplet is $t_k=(x^{(k)},a_0,z_0)$.
Using \eqref{eq:lb-bernoulli-upper} in the direction from the
reference instance to the alternative,
\begin{align}
    \operatorname{kl}(p,p+\alpha)
    &\le\frac{\alpha^2}{(p+\alpha)(1-p-\alpha)}
     \le\frac{2\alpha^2}{p(1-p)}\notag\\
    &\le9\gamma\beta\alpha^2
     \le1152\frac{\eps^2}{\beta^3}.
    \label{eq:tradeoff-branch-kl}
\end{align}
Here we used $\alpha<(1-p)/2$, $p\ge1/3$, and $\gamma\ge1/2$.
Equation~\eqref{eq:tradeoff-testing} implies
\[
    \mathbb E_0\!\left[
        N_{t_k}^{\rm off}+N_{t_k}^{\rm query}(y_k)\right]
    \ge\frac{\kappa}{1152}\frac{\beta^3}{\eps^2}.
\]
All these expectations use the same reference instance. Its offline
phase is independent of which query will subsequently be requested,
and the triplets $t_1,\ldots,t_d$ are distinct. Hence
\[
    \sum_{k=1}^d\mathbb E_0[N_{t_k}^{\rm off}]
    \le n_{\rm learn},
    \qquad
    \sum_{k=1}^d\mathbb E_0[N_{t_k}^{\rm query}(y_k)]
    \le d\,m_{\rm query}.
\]
Summing proves \eqref{eq:offline-main-bound}. Since
$d=S-3\ge S/4$ and $d\le S$,
\begin{equation}
    n_{\rm learn}+S\,m_{\rm query}
    \ge c_2\frac{S\beta^3}{\eps^2},
    \qquad c_2:=\frac{\kappa}{4608}.
    \label{eq:tradeoff-final}
\end{equation}

\paragraph{The query lower bound.}
We now prove
\[
    m_{\rm query}\ge c_1\frac{\beta^2}{\eps^2},
    \qquad c_1:=\frac{\kappa}{384}.
\]
For each $i$, compare $\mathcal M_0$ with the alternative
$(\mathbf0,1,i)$ at the query $y_i=(x_\eta,z^{(i)})$.
Because $x_\beta$ and $x_T$ have values $\beta$ and zero,
respectively, the two pointwise values are
\[
    V_0^\pi(y_i)=\frac{\gamma\beta}{2},
    \qquad
    V_i^\pi(y_i)=\gamma\beta\left(\frac12+\Delta\right).
\]
Their separation is $\gamma\beta\Delta=3\eps>2\eps$.
Only the triplet $s_i=(x_\eta,a_0,z^{(i)})$ differs between the
instances. Again using \eqref{eq:lb-bernoulli-upper},
\begin{equation}
    \operatorname{kl}\!\left(\frac12,\frac12+\Delta\right)
    \le\frac{\Delta^2}{1/4-\Delta^2}
    \le\frac{16}{3}\Delta^2
    \le192\frac{\eps^2}{\beta^2}.
    \label{eq:tradeoff-query-kl}
\end{equation}

Under $\mathcal M_0$, the offline counts satisfy
\[
    \sum_{i=1}^M\mathbb E_0[N_{s_i}^{\rm off}]
    \le n_{\rm learn}\le\bar n.
\]
There is therefore an index $i_\star$ such that
$\mathbb E_0[N_{s_{i_\star}}^{\rm off}]\le\bar n/M$.
This index is fixed once the algorithm is fixed, so
$y_{i_\star}$ is a valid fixed query for the PAC guarantee.
Applying \eqref{eq:tradeoff-testing} and
\eqref{eq:tradeoff-query-kl} at this query gives
\[
    \kappa
    \le192\frac{\eps^2}{\beta^2}
       \left(\frac{\bar n}{M}+m_{\rm query}\right).
\]
Choose $c_Z=384/\kappa$. By the definition of $M$,
\[
    192\frac{\eps^2}{\beta^2}\frac{\bar n}{M}
    \le\frac\kappa2.
\]
It follows that
\[
    m_{\rm query}
    \ge\frac{\kappa}{384}\frac{\beta^2}{\eps^2},
\]
as claimed.

Finally,
\[
    |\Zcal_{\bar n}|=M+1
    \le3+c_Z\frac{\bar n\eps^2}{\beta^2}
    \le C\left(1+\frac{\bar n\eps^2}{\beta^2}\right),
\]
with $C=3+c_Z$. Both lower bounds hold for the same family, which
has exactly $S$ states and one action.
\end{proof}

\paragraph{Consequence for the offline budget.}
The precise consequence of \eqref{eq:tradeoff-final} is
\[
    n_{\rm learn}
    \ge c_2\frac{S\beta^3}{\eps^2}-S\,m_{\rm query}.
\]
In particular, if $m_{\rm query}\le A\beta^2/\eps^2$ for an absolute
constant $A$, then
\[
    n_{\rm learn}
    \ge\frac{S\beta^2}{\eps^2}(c_2\beta-A)
    \ge\frac{c_2}{2}\frac{S\beta^3}{\eps^2}
    \qquad\text{whenever }\beta\ge\frac{2A}{c_2}.
\]
Thus the stated offline order follows in this regime. A uniform
implication for all $\gamma\in[1/2,1)$ requires a corresponding
restriction on the constant in the query budget.

\label{app:unknown-control-lower}

\subsubsection{Best Policy Extraction}
\label{app:bpe-query-lower}

\begin{restateproposition}{prop:bpe-query-lower}[BPE offline/query lower bound]
    \label{prop:lb-bpe-unknown}
Let $\beta\ge256$, $\gamma=1-\beta^{-1}$, $\eps\in(0,1]$,
and $M\ge1$ be an integer.
There is a family of contextual MDPs with two controlled states,
two actions, and $M$ contexts, sharing the same known rewards and
known uniform context distribution, such that any stationary
decision-time mechanism satisfying the pointwise BPE objective
on every instance with failure probability $\delta\le1/6$ obeys
\begin{equation}
 \frac{n_{\rm learn}}{M}+m_{\rm query}
 \ge c_{\rm BPE}\frac{\beta^3}{\eps^2},
 \qquad
 c_{\rm BPE}:=\frac{1}{400\cdot64^2}.
 \label{eq:bpe-query-tradeoff-lower}
\end{equation}
The offline phase may sample adaptively. When answering a query,
the mechanism may also sample adaptively at arbitrary triplets.
\end{restateproposition}

We use the stationary-policy interpretation of the main:
the output is a random deterministic map
$\widehat\Pi:\Xcal\times\Zcal\to\Acal$.
For each fixed pair, its action can be generated from the common
offline transcript and at most $m_{\rm query}$ further oracle calls.
A cached implementation keeps that action on repeated queries.
Only these marginal action-generation guarantees are used below;
independence between the actions at different contexts is unnecessary.

\begin{proof}
We construct a family in which a near-optimal policy must identify
a constant fraction of the context labels.

\paragraph{A family encoding one optimal action per context.}
Let
\[
 \Xcal=\{s,x_T\},\qquad
 \Acal=\{0,1\},\qquad
 \Zcal=\{1,\ldots,M\},\qquad
 \mu(z)=1/M.
\]
State $x_T$ is absorbing with reward zero. At $s$, both actions
have reward one at every context. Put
\[
 p=1-\frac1\beta=\gamma,\qquad
 \Delta=\frac{64\eps}{\beta^2}.
\]
The family is indexed by
$\boldsymbol\theta=(\theta_1,\ldots,\theta_M)\in\{0,1\}^M$.
At the live state, define
\begin{equation}
 P_{\boldsymbol\theta}(s\mid s,a,z)
 =
 \begin{cases}
  p+\Delta,&a=\theta_z,\\
  p-\Delta,&a\ne\theta_z,
 \end{cases}
 \label{eq:bpe-label-construction}
\end{equation}
with the remaining probability assigned to $x_T$.
Because $\Delta\le1/(4\beta)$, all these probabilities belong to
$(0,1)$. Only the transition kernel depends on
$\boldsymbol\theta$.
The optimal action at context $z$ is $\theta_z$.

\paragraph{Policy value controls the fraction of incorrect actions.}
Fix an instance and a deterministic stationary policy $\pi$. Define
\[
 e_z(\pi,\boldsymbol\theta)
 =\mathbf1_{\{\pi(s,z)\ne\theta_z\}},\qquad
 f(\pi,\boldsymbol\theta)
 =\frac1M\sum_{z=1}^M e_z(\pi,\boldsymbol\theta).
\]
Its context-averaged self-loop probability is
$p+\Delta-2\Delta f$, whereas the optimal policy has self-loop
probability $p+\Delta$.
Applying the absorbing-chain formula
\eqref{eq:lb-geometric-value} to these probabilities gives
\[
 U^\star:=\bar V^\star_{\boldsymbol\theta}(s)
       =\frac1{d_\star},\qquad
 U^\pi:=\bar V^\pi_{\boldsymbol\theta}(s)
       =\frac1{d_\pi},
\]
where
\[
 d_\star=1-\gamma(p+\Delta),\qquad
 d_\pi=d_\star+2\gamma\Delta f.
\]
Since $0\le f\le1$ and $\Delta\le1/(4\beta)$,
\[
 0<d_\star\le\frac2\beta,\qquad
 0<d_\pi\le\frac9{4\beta}.
\]
For a fixed initial context $z$, write
$p_z^\pi=p+\Delta-2\Delta e_z$.
The pointwise Bellman identity yields
\begin{align}
 V^\star_{\boldsymbol\theta}(s,z)
       -V^\pi_{\boldsymbol\theta}(s,z)
 &=\gamma(p+\Delta)(U^\star-U^\pi)
       +\gamma(p+\Delta-p_z^\pi)U^\pi\notag\\
 &\ge\gamma(p+\Delta)
       \frac{2\gamma\Delta f}{d_\star d_\pi}\notag\\
 &\ge\frac{\Delta\beta^2}{9}f
  =\frac{64}{9}\eps f.
 \label{eq:bpe-policy-loss-error-fraction}
\end{align}
In the last inequality we used $\gamma\ge1/2$ and
$\gamma(p+\Delta)\ge1/2$, both valid for $\beta\ge256$.
This lower bound holds for every initial context.

Apply \eqref{eq:bpe-policy-loss-error-fraction} to the realized
random policy $\widehat\Pi$. Choose one fixed initial context,
say $z_0=1$. Pointwise PAC at $(s,z_0)$ implies, for every instance,
\[
 \Pr_{\boldsymbol\theta}\left\{
 f(\widehat\Pi,\boldsymbol\theta)>9/64\right\}
 \le\delta.
\]
As $f\in[0,1]$, it follows that
\begin{equation}
 \E_{\boldsymbol\theta}
       f(\widehat\Pi,\boldsymbol\theta)
 \le\frac9{64}+\delta
 \le\frac{59}{192}.
 \label{eq:bpe-pac-implies-bit-accuracy}
\end{equation}
No simultaneous PAC event over contexts is required.

\paragraph{Information available for identifying one label.}
Place independent uniform priors on the labels:
$\boldsymbol\Theta\sim\operatorname{Unif}(\{0,1\}^M)$.
Let $\mathbb E$ denote expectation over this prior and the
algorithm's randomness.
For each $z$, let $N_z$ be the number of offline transition calls
of the form $(s,a,z)$, with either action.
Pointwise in the offline transcript,
\[
 \sum_{z=1}^M N_z\le n_{\rm learn}.
\]

For the query $(s,z)$, let $\mathsf P_z^b$ be the law of its
offline-plus-query transcript conditional on $\Theta_z=b$,
with all other labels averaged under their common uniform prior.
Include the algorithm's output action and internal randomization
in this transcript.
The output action is a test of $\Theta_z$.

First fix the other labels and compare the two resulting instances.
Only calls to $(s,a,z)$ have different response laws.
For either action, those laws are Bernoulli with parameters
$p+\Delta$ and $p-\Delta$, in one order or the other.
By the previously established inequality
\eqref{eq:lb-bernoulli-upper},
\begin{equation}
 \max\{
 \operatorname{kl}(p-\Delta,p+\Delta),
 \operatorname{kl}(p+\Delta,p-\Delta)
 \}
 \le16\beta\Delta^2.
 \label{eq:bpe-one-sample-kl}
\end{equation}
Indeed, either denominator $q(1-q)$ is at least $1/(4\beta)$
and the parameter difference is $2\Delta$.

The adaptive likelihood-ratio identity used in the proof of
Lemma~1 of \citet[Appendix~A.1, equation~(18)]{kaufmann2016complexity}
therefore bounds the transcript divergence by
$16\beta\Delta^2$ times the expected number of informative calls.
The number made while answering this query is at most
$m_{\rm query}$, even when arbitrary triplets are queried.
Average this inequality over the other labels and use convexity
of relative entropy under mixtures with the same mixing weights.
We obtain
\begin{equation}
 \operatorname{KL}(\mathsf P_z^0\Vert\mathsf P_z^1)
 \le16\beta\Delta^2
   \left(\mathbb E[N_z\mid\Theta_z=0]+m_{\rm query}\right).
 \label{eq:bpe-transcript-kl}
\end{equation}
Equivalently, one may retain the other labels in the transcript,
apply the identity conditionally, and then discard them by
data processing. Conditioning before applying the identity
accounts for adaptively acquired information about other labels.

\paragraph{Averaging the testing bounds over contexts.}
The equal-prior binary testing inequality gives
\[
 \frac12\left[
   \mathsf P_z^0\{\widehat\Pi(s,z)=1\}
   +\mathsf P_z^1\{\widehat\Pi(s,z)=0\}\right]
 \ge\frac12\left(
       1-\|\mathsf P_z^0-\mathsf P_z^1\|_{\rm TV}\right).
\]
Use Pinsker's inequality in the form recorded in
\citet[Remark~2]{kaufmann2016complexity}, followed by
\eqref{eq:bpe-transcript-kl} and Jensen's inequality:
\begin{align}
 \frac1M\sum_{z=1}^M
       \|\mathsf P_z^0-\mathsf P_z^1\|_{\rm TV}
 &\le
 \left(\frac1{2M}\sum_{z=1}^M
       \operatorname{KL}(\mathsf P_z^0\Vert\mathsf P_z^1)
 \right)^{1/2}\notag\\
 &\le
 \sqrt{8\beta\Delta^2
       \left(\frac{2n_{\rm learn}}{M}+m_{\rm query}\right)}.
 \label{eq:bpe-average-testing-distance}
\end{align}
Here the factor two in the offline term follows from
\[
 \mathbb E[N_z\mid\Theta_z=0]
 =2\mathbb E[N_z\mathbf1_{\{\Theta_z=0\}}]
 \le2\mathbb E N_z,
 \qquad
 \sum_z\mathbb E N_z\le n_{\rm learn}.
\]
This averaging step is valid for adaptive offline sampling.

Consequently,
\begin{equation}
 \mathbb E f(\widehat\Pi,\boldsymbol\Theta)
 \ge\frac12\left[
 1-\sqrt{8\beta\Delta^2
       \left(\frac{2n_{\rm learn}}{M}+m_{\rm query}\right)}
 \right].
 \label{eq:bpe-bit-error-lower}
\end{equation}
On the other hand, averaging
\eqref{eq:bpe-pac-implies-bit-accuracy} over the same prior gives
$\mathbb E f\le59/192$.
Combining the two inequalities and using
$2n_{\rm learn}/M+m_{\rm query}
\le2(n_{\rm learn}/M+m_{\rm query})$ yields
\[
 \frac{n_{\rm learn}}{M}+m_{\rm query}
 \ge\frac{(37/96)^2}{16\beta\Delta^2}
 \ge\frac1{400\beta\Delta^2}
 =\frac1{400\cdot64^2}\frac{\beta^3}{\eps^2}.
\]
This proves \eqref{eq:bpe-query-tradeoff-lower}.
\end{proof}

\begin{corollary}[Query cost with a context-independent offline budget]
\label{cor:bpe-context-independent-query-lower}
Fix $\beta\ge256$, $\eps\in(0,1]$, and a finite offline cap
$\bar n\ge0$. There is a family with two controlled states and
two actions such that any stationary mechanism satisfying BPE
with $\delta\le1/6$ and $n_{\rm learn}\le\bar n$ must use
\[
 m_{\rm query}\ge
 \frac{c_{\rm BPE}}2\frac{\beta^3}{\eps^2}.
\]
\end{corollary}

\begin{proof}
Apply Proposition~\ref{prop:lb-bpe-unknown} with
\[
 M=\left\lceil
       \frac{2\bar n\eps^2}{c_{\rm BPE}\beta^3}
    \right\rceil+1.
\]
Then $n_{\rm learn}/M\le
c_{\rm BPE}\beta^3/(2\eps^2)$, and rearrange
\eqref{eq:bpe-query-tradeoff-lower}.
\end{proof}

This establishes the cubic dependence on $\beta$ and quadratic
dependence on $1/\eps$ of the worst-case BPE query cost for fixed
state and action cardinalities and constant confidence.
It does not establish a quartic offline lower bound.

\subsubsection{Best Value estimation}
\label{app:bve-pe-lower-transfer}

\begin{corollary}[BVE offline/query lower bounds]
\label{cor:lb-bve-unknown}
Fix $\gamma\in[1/2,1)$, $\beta=(1-\gamma)^{-1}$, $S\ge4$,
$\eps\in(0,\beta/24)$, and $\bar n\in\mathbb N$.
The one-action family used in the proof of proposition~\ref{prop:lb-pe-unknown}
also satisfies the following statement.
Any decision-time mechanism achieving BVE uniformly over that
family with $\delta\le1/12$, $n_{\rm learn}\le\bar n$, and
query budget $m_{\rm query}$ must satisfy
\begin{equation}
 m_{\rm query}\ge c_1\frac{\beta^2}{\eps^2},
 \qquad
 n_{\rm learn}+S\,m_{\rm query}
 \ge c_2\frac{S\beta^3}{\eps^2},
 \label{eq:bve-pe-tradeoff-transfer}
\end{equation}
with the same universal constants as in the PE lower bound.
The family has $S$ controlled states and
$|\Zcal|\le C(1+\bar n\eps^2/\beta^2)$ contexts.
\end{corollary}

\begin{proof}
Every instance in the construction of
Proposition~\ref{prop:lb-pe-unknown} has
$\Acal=\{a_0\}$. Thus its unique policy $\pi_0\equiv a_0$
satisfies
\[
 V^\star(x,z)=V^{\pi_0}(x,z)
 \qquad\text{for every }(x,z).
\]
Run any BVE mechanism on this family and use
$\widehat{\Vcal}^\star(x,z)$ as the PE estimate for $\pi_0$.
Its error event and both oracle budgets are unchanged.
Proposition~\ref{prop:lb-pe-unknown} therefore applies directly,
proving both inequalities in \eqref{eq:bve-pe-tradeoff-transfer}.
No new testing construction is needed.
\end{proof}

The precise offline consequence of
\eqref{eq:bve-pe-tradeoff-transfer} is
\[
 n_{\rm learn}\ge
 c_2\frac{S\beta^3}{\eps^2}-S\,m_{\rm query}.
\]
In particular, if
$m_{\rm query}\le C_q\beta^2/\eps^2$ for a fixed constant $C_q$,
then
\[
 n_{\rm learn}\ge
 \frac{c_2}{2}\frac{S\beta^3}{\eps^2}
 \qquad\text{whenever }\beta\ge2C_q/c_2.
\]
These are the PE-equivalent BVE lower bounds; they leave open
the optimal offline complexity of BVE.
\subsection{Upper Bounds}
\subsubsection{Best Value Estimation and Best Policy Extraction}
\label{app:unknown-control-upper}

We prove the BVE and BPE parts of Theorem~\ref{thm:main-pe-Punknown}.
Write $S=|\Xcal|$, $A=|\Acal|$, and $\beta=(1-\gamma)^{-1}$.
Throughout, $\gamma\in[1/2,1)$, $\eps\in(0,1]$, and
$\delta\in(0,1/4)$. Rewards are known and take values in $[0,1]$.
The learner can draw independent contexts from $\mu$ and independent
successors from $P(\cdot\mid x,a,z)$ at any chosen triplet.
Sampled contexts can be stored and reused.
The context space $\Zcal$ is finite, as in the model definition.

For a continuation vector $\bar v\in\mathbb R^S$, recall
\[
 q_{\bar v}(x,a,z)
 =r(x,a,z)+\gamma\sum_{x'}P(x'\mid x,a,z)\bar v(x'),
 \qquad
 (T\bar v)(x)=\E_{Z\sim\mu}\max_a q_{\bar v}(x,a,Z).
\]
Set $U=\bar V^\star$, the fixed point of $T$. Then
$0\le U\le\beta\mathbf1$ and
$V^\star(x,z)=\max_a q_U(x,a,z)$.
For a deterministic stationary policy $\pi$, define
\[
 \bar P^\pi(x,x')
 =\E_{Z\sim\mu}P(x'\mid x,\pi(x,Z),Z),
 \qquad
 \mathcal R_\pi=(I-\gamma\bar P^\pi)^{-1}.
\]
The matrix $\mathcal R_\pi$ is nonnegative and
$\mathcal R_\pi\mathbf1=\beta\mathbf1$.
Vector inequalities, squares, and square roots below are coordinatewise.

\paragraph{Common offline upper bound}
\label{app:unknown-control-offline}

\begin{proposition}[Offline preparation for BVE and BPE]
\label{prop:unknown-control-offline}
Algorithm~\ref{alg:unknown-control-offline} uses
\[
 n_{\rm learn}
 =\widetilde O\!\left(
 SA\,[S+\log(1/\delta)]\,\frac{\beta^4}{\eps^2}
 \right)
\]
context and transition calls in total and returns a vector
$\bar v\in[0,\beta]^S$ such that
\begin{equation}
 \Pr\!\left\{
 0\le\bar v\le T\bar v,\quad
 \|U-\bar v\|_\infty\le\eps/4
 \right\}\ge1-\delta/3.
 \label{eq:unknown-control-offline-event}
\end{equation}
The hidden logarithms involve only $S,A,\beta,1/\eps,1/\delta$.
\end{proposition}

We construct one uniform multilevel operator and reuse it throughout
the offline value iterations. The two-half correction below is the
antithetic construction used for nested maxima by
\citet{giles2019decision}. We establish its uniform high-probability
accuracy over continuation vectors, then prove directly that the
resulting iterations preserve a subsolution and converge.

\paragraph{Multilevel samples.}
Let $\mathcal D=[0,\beta]^S$. A level-$\ell$ replicate consists of
one context $Z\sim\mu$ and $2^\ell$ successors for every $(x,a)$,
conditionally independently given $Z$. Define
\[
 G_\ell^x(\bar v)=
 \max_a\left\{
 r(x,a,Z)+\frac{\gamma}{2^\ell}
 \sum_{j=1}^{2^\ell}\bar v(X_j^{x,a})
 \right\}.
\]
For $\ell\ge1$, form $G_{\ell-1}^{x,(1)}$ and
$G_{\ell-1}^{x,(2)}$ from the two halves of this same batch,
using the same context. Set
\[
 D_0^x=G_0^x,\qquad
 D_\ell^x=G_\ell^x
 -\tfrac12G_{\ell-1}^{x,(1)}
 -\tfrac12G_{\ell-1}^{x,(2)}.
\]
Different replicates are independent. A context is shared across
all states and actions within a replicate. In particular,
\begin{equation}
 \sum_{\ell=0}^L\E D_\ell^x(\bar v)=\E G_L^x(\bar v).
 \label{eq:unknown-control-telescope}
\end{equation}

\begin{lemma}[Bias, variance, and continuity of the level corrections]
\label{lem:unknown-control-levels}
For every $\bar v\in\mathcal D$, state $x$, and $L\ge0$,
\begin{align}
 0\le\E G_L^x(\bar v)-(T\bar v)(x)
 &\le\gamma\beta
 \sqrt{\frac{2\log(2A)}{2^L}},
 \label{eq:unknown-control-bias}\\
 \operatorname{Var}(D_0^x(\bar v))&\le\beta^2,
 \notag\\
 \operatorname{Var}(D_\ell^x(\bar v))
 &\le8\beta^2\log(2A)\,2^{-\ell},
 \qquad \ell\ge1.
 \label{eq:unknown-control-level-variance}
\end{align}
For every realization of the samples, the maps $D_0^x$ and
$D_\ell^x$, $\ell\ge1$, are respectively $\gamma$-Lipschitz
and $2\gamma$-Lipschitz in the supremum norm.
\end{lemma}

\begin{proof}
Conditional on $Z$, an empirical action value at level $L$ is
unbiased and its centered error is sub-Gaussian with variance
proxy at most $\gamma^2\beta^2/2^L$, by the bounded-variable
exponential-moment bound used in the proof of
\citet[Theorem~2]{hoeffding1963probability}.
Jensen's inequality gives the nonnegative bias. The exponential-moment
bound for the maximum of $A$ such errors gives the upper bound
in \eqref{eq:unknown-control-bias}.

For a positive level, let $Q_1,Q_2\in\mathbb R^A$ be the two
half-batch action-value vectors and write $f(q)=\max_a q_a$.
The full-batch vector is $(Q_1+Q_2)/2$, so
\[
 0\le-D_\ell^x(\bar v)
 =\tfrac12 f(Q_1)+\tfrac12 f(Q_2)
   -f\!\left(\tfrac{Q_1+Q_2}{2}\right)
 \le\tfrac12\|Q_1-Q_2\|_\infty.
\]
Conditional on $Z$, each coordinate of $Q_1-Q_2$ is centered
and sub-Gaussian with variance proxy
$4\gamma^2\beta^2 2^{-\ell}$.
The union bound and integration of its tail imply
\[
 \E\!\left[\|Q_1-Q_2\|_\infty^2\mid Z\right]
 \le32\gamma^2\beta^2\log(2A)\,2^{-\ell}.
\]
Consequently, $\operatorname{Var}(D_\ell^x)\le\E[(D_\ell^x)^2]$
gives \eqref{eq:unknown-control-level-variance}, including the
randomness of $Z$. Also $G_0^x(\bar v)\in[0,\beta]$, proving
the level-zero bound.
Finally, each map $G_\ell^x$ is $\gamma$-Lipschitz.
The claimed correction bounds follow by the triangle inequality.
\end{proof}

\paragraph{A uniform estimator from stored samples.}
Fix a target accuracy $\eta\in(0,1/2]$ and failure probability
$\zeta\in(0,1/4)$. Define
\begin{align}
 L&=\left\lceil\log_2
       \frac{32\beta^2\log(2A)}{\eta^2}\right\rceil,
 &t_0&=\frac{\eta}{4(L+1)},
 &h&=\frac{\eta}{16(L+1)},\notag\\
 \Lambda&=\log\frac{2S(L+1)}{\zeta}
       +S\log\left(1+\frac{32\beta(L+1)}{\eta}\right),\notag\\
 \nu_0^2&=\beta^2,
 &\nu_\ell^2&=8\beta^2\log(2A)\,2^{-\ell}
 \quad(\ell\ge1).
 \label{eq:unknown-control-mlmc-parameters}
\end{align}
Take
\begin{equation}
 B=2\lceil4\Lambda\rceil+1,\qquad
 q_\ell=\max\left\{1,\left\lceil
             \frac{16\nu_\ell^2}{t_0^2}\right\rceil\right\}.
 \label{eq:unknown-control-mlmc-blocks}
\end{equation}
For each level, draw $B$ independent blocks of $q_\ell$ replicates.
For every input vector, evaluate the same stored samples and set
\[
 \widehat d_\ell^x(\bar v)
 =\operatorname{median}_{1\le b\le B}
   \left\{\frac1{q_\ell}
     \sum_{j=1}^{q_\ell}D_{\ell,b,j}^x(\bar v)\right\},
 \qquad
 (\widehat T\bar v)(x)=\sum_{\ell=0}^L\widehat d_\ell^x(\bar v).
\]

\begin{lemma}[Uniform multilevel operator]
\label{lem:unknown-control-uniform}
With probability at least $1-\zeta$,
\begin{equation}
 \sup_{\bar v\in\mathcal D}
 \|\widehat T\bar v-T\bar v\|_\infty\le\eta.
 \label{eq:unknown-control-uniform}
\end{equation}
For a universal constant $C$, the construction uses
\begin{align}
 n_\mu&=B\sum_{\ell=0}^Lq_\ell
 \le C\frac{\beta^2}{\eta^2}\log(2A)(L+1)^2\Lambda,
 \label{eq:unknown-control-context-count}\\
 n_P&=SA B\sum_{\ell=0}^L2^\ell q_\ell
 \le CSA\frac{\beta^2}{\eta^2}\log(2A)(L+1)^3\Lambda.
 \label{eq:unknown-control-transition-count}
\end{align}
\end{lemma}

\begin{proof}
Let $\mathcal N$ be an $h$-net of $\mathcal D$ in the supremum norm
with $|\mathcal N|\le(1+2\beta/h)^S$.
For a fixed $\bar w\in\mathcal N$, level $\ell$, and state $x$,
the median-of-means inequality of
\citet[Theorem~2]{lugosi2019mean}, together with
Lemma~\ref{lem:unknown-control-levels}, gives
\[
 \Pr\!\left\{
 |\widehat d_\ell^x(\bar w)-\E D_\ell^x(\bar w)|>t_0
 \right\}
 \le e^{-B/8}.
\]
Indeed, that theorem gives error at most $2\nu_\ell/\sqrt{q_\ell}$,
which is at most $t_0$ with the chosen block size.
Since $B\ge8\Lambda$, a union bound over
$S(L+1)|\mathcal N|$ triples gives simultaneous accuracy with
probability at least $1-\zeta$.
On this event, telescoping and
Lemma~\ref{lem:unknown-control-levels} imply, for every $\bar w\in\mathcal N$,
\[
 \|\widehat T\bar w-T\bar w\|_\infty
 \le(L+1)t_0+
      \gamma\beta\sqrt{\frac{2\log(2A)}{2^L}}
 \le\eta/2.
\]

The median is $1$-Lipschitz in the maximum norm of its inputs.
Thus $\widehat T$ has Lipschitz constant at most
$\gamma(1+2L)\le2(L+1)$, whereas $T$ has constant $\gamma\le1$.
For any $\bar v\in\mathcal D$, choose $\bar w\in\mathcal N$ with
$\|\bar v-\bar w\|_\infty\le h$. Then
\[
 \|\widehat T\bar v-T\bar v\|_\infty
 \le\eta/2+[2(L+1)+1]h<\eta.
\]
This proves \eqref{eq:unknown-control-uniform}.
The net is only an analytical device and is never enumerated.

Each replicate costs one context call and $SA2^\ell$ transition
calls, which gives the two exact counts. Substituting $q_\ell$,
using $\sum_{\ell\ge1}2^{-\ell}\le1$ and
$2^L\le64\beta^2\log(2A)/\eta^2$, proves
\eqref{eq:unknown-control-context-count}--%
\eqref{eq:unknown-control-transition-count}.
\end{proof}

\begin{algorithm}[t]
\caption{Offline learning for BVE and BPE with unknown $P$ and $\mu$}
\label{alg:unknown-control-offline}
\begin{algorithmic}[1]
\Require Known rewards $r$; sampling oracles for $\mu,P$;
         $\gamma,\eps,\delta$.
\Ensure A continuation vector $\bar v\in[0,\beta]^S$.
\State $\beta\gets(1-\gamma)^{-1}$,
       $\eta\gets\eps/(16\beta)$, $\zeta\gets\delta/3$,
       $J\gets\lceil\beta\log(8\beta/\eps)\rceil$.
\State Set $L,B,q_0,\ldots,q_L$ by
       \eqref{eq:unknown-control-mlmc-parameters}--\eqref{eq:unknown-control-mlmc-blocks}.
\For{$\ell=0,\ldots,L$}
  \For{$b=1,\ldots,B$ and $j=1,\ldots,q_\ell$}
    \State Draw $Z_{\ell,b,j}\sim\mu$ independently of all preceding batches.
    \For{each $(x,a)\in\Xcal\times\Acal$}
      \State Draw $2^\ell$ independent successors from
             $P(\cdot\mid x,a,Z_{\ell,b,j})$.
    \EndFor
    \State Store this batch, defining the functions
           $D_{\ell,b,j}^x(\cdot)$ for all $x$.
  \EndFor
\EndFor
\State Set $\bar v^{(0)}\gets0$.
\For{$k=0,\ldots,J-1$}
  \State Evaluate $\widehat T\bar v^{(k)}$ using the stored batches.
  \State $\bar v^{(k+1)}\gets
    \min\{\beta\mathbf1,
    \max\{\bar v^{(k)},\widehat T\bar v^{(k)}-\eta\mathbf1\}\}$
    \Comment{Coordinatewise}
\EndFor
\State \Return $\bar v^{(J)}$.
\end{algorithmic}
\end{algorithm}

\begin{proof}[Proof of Proposition~\ref{prop:unknown-control-offline}]
Apply Lemma~\ref{lem:unknown-control-uniform} with
$\eta=\eps/(16\beta)$ and $\zeta=\delta/3$.
Let $\mathcal E_{\rm off}$ denote its uniform success event, which
has probability at least $1-\delta/3$.
On this event, for every $\bar w\in\mathcal D$,
\begin{equation}
 T\bar w-2\eta\mathbf1
 \le\widehat T\bar w-\eta\mathbf1
 \le T\bar w.
 \label{eq:unknown-control-lower-backup}
\end{equation}

We first show inductively that every iterate satisfies
\[
 0\le\bar v^{(k)}\le T\bar v^{(k)},\qquad
 \bar v^{(k)}\le U.
\]
These properties hold at $k=0$ because rewards are nonnegative.
If they hold at $k$, both arguments of the coordinatewise maximum
in the update are at most $T\bar v^{(k)}\le TU=U$.
Thus upper clipping is inactive on $\mathcal E_{\rm off}$, and
\[
 0\le\bar v^{(k)}\le\bar v^{(k+1)}
 \le T\bar v^{(k)}\le U.
\]
Monotonicity of $T$ also yields
$\bar v^{(k+1)}\le T\bar v^{(k)}\le T\bar v^{(k+1)}$,
completing the induction. On all sample outcomes, clipping and
the monotone update keep the iterates in $\mathcal D$.

By \eqref{eq:unknown-control-lower-backup},
$\bar v^{(k+1)}\ge T\bar v^{(k)}-2\eta\mathbf1$ on
$\mathcal E_{\rm off}$. Since $T$ is a $\gamma$-contraction,
\[
 \|U-\bar v^{(k+1)}\|_\infty
 \le\gamma\|U-\bar v^{(k)}\|_\infty+2\eta.
\]
Iterating this inequality gives
\[
 \|U-\bar v^{(J)}\|_\infty
 \le\gamma^J\beta+2\beta\eta
 \le\frac{\eps}{8}+\frac{\eps}{8}
 =\frac{\eps}{4},
\]
where $\gamma\le e^{-1/\beta}$ and
$J\ge\beta\log(8\beta/\eps)$.
This proves \eqref{eq:unknown-control-offline-event}.
Uniformity makes this argument valid for all data-dependent iterates
using the same stored samples; no union bound over iterations is needed.

Only the construction of $\widehat T$ uses oracle calls.
Equations~\eqref{eq:unknown-control-context-count}--%
\eqref{eq:unknown-control-transition-count} therefore give
\begin{equation}
 n_{\rm learn}=n_\mu+n_P
 \le CSA\frac{\beta^4}{\eps^2}
       \log(2A)(L+1)^3\Lambda,
 \label{eq:unknown-control-offline-count}
\end{equation}
after adjusting the universal constant $C$.
Here
\[
 L=\left\lceil\log_2
       \frac{8192\beta^4\log(2A)}{\eps^2}\right\rceil,
 \qquad
 \Lambda=\log\frac{6S(L+1)}{\delta}
       +S\log\left(1+\frac{512\beta^2(L+1)}{\eps}\right).
\]
These expressions yield the stated $\widetilde O$ bound and contain
no dependence on $|\Zcal|$. The net used in the proof of
Lemma~\ref{lem:unknown-control-uniform} is never enumerated.
The number of stored observations and iterations is polynomial in
$S,A,\beta,1/\eps,\log(1/\delta)$; hence the arithmetic cost is also
polynomial when rewards can be evaluated efficiently.
\end{proof}

\paragraph{Best Value Estimation}
\label{app:unknown-control-bve}

\begin{proposition}[BVE upper bound]
\label{prop:unknown-control-bve}
With the offline budget of Proposition~\ref{prop:unknown-control-offline},
BVE is achieved using
\[
 m_{\rm query}^{\rm BVE}
 =A\left\lceil
       \frac{2\beta^2}{\eps^2}\log\frac{6A}{\delta}
    \right\rceil.
\]
In particular, for every fixed $(x,z)$,
\[
 \Pr\!\left\{
 |\widehat{\Vcal}^\star(x,z)-V^\star(x,z)|>\eps
 \right\}\le\delta.
\]
\end{proposition}

\begin{proof}
At the query, independently of the offline data, draw
$m_{\rm val}=\lceil2\beta^2\eps^{-2}\log(6A/\delta)\rceil$
successors per action and return
\[
 \widehat{\Vcal}^\star(x,z)
 =\max_a\widehat q_a,\qquad
 \widehat q_a
 =r(x,a,z)+\frac{\gamma}{m_{\rm val}}
   \sum_{j=1}^{m_{\rm val}}\bar v(X_{a,j}).
\]
Conditional on the offline data, the summands lie in $[0,\gamma\beta]$.
Hoeffding's inequality \citep[Theorem~2]{hoeffding1963probability}
and a union bound over actions give
\[
 \Pr\!\left\{
 \max_a|\widehat q_a-q_{\bar v}(x,a,z)|>\eps/2
 \,\middle|\,\text{offline data}\right\}
 \le2A\exp\left(-\frac{m_{\rm val}\eps^2}{2\beta^2}\right)
 \le\delta/3.
\]
On this event's complement and the offline event
\eqref{eq:unknown-control-offline-event},
\[
 |\widehat{\Vcal}^\star(x,z)-V^\star(x,z)|
 \le\eps/2+\gamma\|\bar v-U\|_\infty
 \le3\eps/4.
\]
The combined failure probability is at most $2\delta/3\le\delta$.
The query uses $Am_{\rm val}$ transition calls and no context calls.
\end{proof}

\paragraph{Best Policy Extraction}
\label{app:unknown-control-bpe}

\begin{proposition}[BPE upper bound]
\label{prop:unknown-control-bpe}
With the offline budget of Proposition~\ref{prop:unknown-control-offline},
the BPE procedure in Algorithm~\ref{alg:unknown-control-query}
satisfies, for every
fixed $(x,z)$,
\[
 \Pr\!\left\{
 V^\star(x,z)-V^{\widehat\Pi}(x,z)>\eps
 \right\}\le\delta,
\]
using at most
\begin{equation}
 m_{\rm query}^{\rm BPE}
 =A\left\lceil
     \frac{4096\beta^3}{\eps^2}
     \log\frac{1500AS\beta^2}{\delta\eps}
   \right\rceil
 \label{eq:unknown-control-bpe-budget}
\end{equation}
transition calls at each previously unseen pair. Repeated pairs use
no additional calls. No context calls are made at query time.
\end{proposition}

\paragraph{The query rule and its policy interpretation.}
Set
\[
 \rho_q=\frac{\delta\eps}{300S\beta^2},\qquad
 \ell_q=\log\frac{5A}{\rho_q}
       =\log\frac{1500AS\beta^2}{\delta\eps},\qquad
 m=m_{\rm pol}
   =\left\lceil4096\beta^3\eps^{-2}\ell_q\right\rceil.
\]
At the first query to $(x,z)$, draw $m$ independent successors per
action and put
\[
 Y_{a,j}=\gamma\bar v(X_{a,j}),\qquad
 \bar Y_a=\frac1m\sum_{j=1}^mY_{a,j},\qquad
 \widehat q_a=r(x,a,z)+\bar Y_a,\qquad
 \widehat s_a^2=\frac1{m-1}\sum_{j=1}^m(Y_{a,j}-\bar Y_a)^2.
\]
Choose, with a fixed deterministic tie-breaking rule,
\begin{equation}
 \widehat\Pi(x,z)\in
 \arg\max_a\{\widehat q_a+b_a\},\qquad
 b_a=\sqrt{\frac{2\widehat s_a^2\ell_q}{m}}
       +\frac{7\beta\ell_q}{3(m-1)},
 \label{eq:unknown-control-optimistic-rule}
\end{equation}
and cache the action. The last term in $b_a$ is common to all actions
and may be omitted when computing the maximizer.

Conditional on the offline data, assign mutually independent
transition-sample batches to all pairs $(x,z)$, independently of
the randomness of the MDP trajectory used to evaluate the policy.
Fixing this table defines a deterministic stationary map
$\widehat\Pi:\Xcal\times\Zcal\to\Acal$ through
\eqref{eq:unknown-control-optimistic-rule}.
The BPE probability is taken over the offline data and this entire
query-sample table; $V^{\widehat\Pi}$ averages only the environment
trajectory for the realized policy.

This table is an analytical coupling. Algorithm~\ref{alg:unknown-control-query}
generates its entries only at the first query to a pair and caches
the resulting action. A pair's action uses only the offline vector
and its own batch. It therefore has the same law as the preassigned
policy, with the budget in \eqref{eq:unknown-control-bpe-budget}
at each new pair and zero calls at repeated pairs.

\begin{algorithm}[t]
\caption{BVE and BPE query procedures}
\label{alg:unknown-control-query}
\begin{algorithmic}[1]
\Require Offline vector $\bar v$ from Algorithm~\ref{alg:unknown-control-offline};
         query $(x,z)$; objective $H\in\{\mathrm{BVE},\mathrm{BPE}\}$.
\Require Known rewards $r$; transition oracle $P$; parameters $\beta,\eps,\delta$.
\Statex \textbf{Persistent state:} an initially empty action cache $\mathcal C$.
\If{$H=\mathrm{BPE}$ and $(x,z)\in\operatorname{dom}(\mathcal C)$}
  \State \Return $\mathcal C(x,z)$.
\EndIf
\If{$H=\mathrm{BVE}$}
  \State $m\gets\lceil2\beta^2\eps^{-2}\log(6A/\delta)\rceil$.
\Else
  \State $\ell_q\gets\log(1500AS\beta^2/(\delta\eps))$,
         $m\gets\lceil4096\beta^3\eps^{-2}\ell_q\rceil$.
\EndIf
\For{each $a\in\Acal$}
  \State Draw $m$ independent successors
         $X_{a,1},\ldots,X_{a,m}\sim P(\cdot\mid x,a,z)$
         using fresh oracle randomness.
  \State $Y_{a,j}\gets\gamma\bar v(X_{a,j})$ for $j=1,\ldots,m$;
         $\bar Y_a\gets m^{-1}\sum_{j=1}^mY_{a,j}$.
  \State $\widehat q_a\gets r(x,a,z)+\bar Y_a$.
  \If{$H=\mathrm{BPE}$}
    \State $\widehat s_a^2\gets(m-1)^{-1}
           \sum_{j=1}^m(Y_{a,j}-\bar Y_a)^2$.
    \State $b_a\gets\sqrt{2\widehat s_a^2\ell_q/m}
           +7\beta\ell_q/[3(m-1)]$.
  \EndIf
\EndFor
\If{$H=\mathrm{BVE}$}
  \State \Return $\max_a\widehat q_a$.
\Else
  \State $\widehat a\gets\arg\max_a(\widehat q_a+b_a)$,
         with fixed deterministic tie-breaking.
  \State $\mathcal C(x,z)\gets\widehat a$.
  \State \Return $\widehat a$.
\EndIf
\end{algorithmic}
\end{algorithm}

\begin{lemma}[Local loss controlled by the selected action's variance]
\label{lem:unknown-control-selected-variance}
Fix any offline vector $\bar v\in[0,\beta]^S$ and any pair $(x,z)$.
Write
\[
 t_q=\ell_q/m,\qquad
 s_a(x,z)=\gamma
 \sqrt{\operatorname{Var}_{P(\cdot\mid x,a,z)}(\bar v(X'))},
 \qquad
 D(x,z)=\max_a q_{\bar v}(x,a,z)
          -q_{\bar v}(x,\widehat\Pi(x,z),z).
\]
There is a query-sample event $\mathcal E_{x,z}$ of probability at
least $1-\rho_q$ on which
\begin{equation}
 D(x,z)\le
 3\sqrt{t_q}\,s_{\widehat\Pi(x,z)}(x,z)+16\beta t_q.
 \label{eq:unknown-control-local-loss}
\end{equation}
\end{lemma}

\begin{proof}
Apply Theorem~4 of \citet{maurer2009} to $Y_{a,j}/\beta$
and to $1-Y_{a,j}/\beta$.
Together they imply $|\widehat q_a-q_{\bar v}(x,a,z)|\le b_a$,
except with probability at most $4e^{-\ell_q}$.
The upper-tail inequality in their Theorem~10, equation~(4),
also gives
\[
 \widehat s_a\le s_a(x,z)
       +\beta\sqrt{\frac{2\ell_q}{m-1}},
\]
except with probability at most $e^{-\ell_q}$.
Let $\mathcal E_{x,z}$ be the intersection of these events over
all actions. Its failure probability is at most
$5Ae^{-\ell_q}=\rho_q$.

On $\mathcal E_{x,z}$, if $a^\star$ maximizes the true action values,
the optimistic selection rule gives
\[
 q_{\bar v}(x,a^\star,z)
 \le\widehat q_{a^\star}+b_{a^\star}
 \le\widehat q_{\widehat\Pi(x,z)}+b_{\widehat\Pi(x,z)}
 \le q_{\bar v}(x,\widehat\Pi(x,z),z)
       +2b_{\widehat\Pi(x,z)}.
\]
Since $m\ge2$, the standard-deviation bound implies
\[
 b_a\le\sqrt{2t_q}\,s_a(x,z)
       +\left(2\sqrt2+\frac{14}{3}\right)\beta t_q
 \le\sqrt{2t_q}\,s_a(x,z)+8\beta t_q.
\]
Now use $2\sqrt2\le3$.
\end{proof}

The next bound holds for every policy, including the policy selected
from the query samples. This avoids any union bound over policies.

\begin{lemma}[Discounted propagation of continuation-value variance]
\label{lem:unknown-control-resolvent}
For a deterministic stationary policy $\pi$ and vector $f$, let
\[
 \sigma_{\pi,f}^2
 =\bar P^\pi f^2-(\bar P^\pi f)^2.
\]
Then
\begin{equation}
 \|\mathcal R_\pi(\gamma\sigma_{\pi,U})\|_\infty
 \le\sqrt{2\beta^3}.
 \label{eq:unknown-control-optimal-resolvent}
\end{equation}
Consequently, if $\|\bar v-U\|_\infty\le e$, then
\begin{equation}
 \|\mathcal R_\pi(\gamma\sigma_{\pi,\bar v})\|_\infty
 \le\sqrt{2\beta^3}+\gamma\beta e.
 \label{eq:unknown-control-approx-resolvent}
\end{equation}
\end{lemma}

\begin{proof}
In this proof write $P_\pi=\bar P^\pi$ and $\mathcal R=\mathcal R_\pi$.
Optimality and nonnegative rewards imply
$0\le\gamma P_\pi U\le T^\pi U\le TU=U\le\beta\mathbf1$.
Hence
\begin{align*}
 \gamma^2\sigma_{\pi,U}^2
 &= (\gamma^2P_\pi-I)U^2
       +U^2-(\gamma P_\pi U)^2\\
 &\le(\gamma^2P_\pi-I)U^2
       +2\beta(U-\gamma P_\pi U).
\end{align*}
Apply the nonnegative matrix $\mathcal R$ and use
\[
 \mathcal R(\gamma^2P_\pi-I)
 =-I-\gamma(1-\gamma)\mathcal R P_\pi,\qquad
 \mathcal R(U-\gamma P_\pi U)=U.
\]
Discarding nonpositive terms gives
\[
 \mathcal R(\gamma^2\sigma_{\pi,U}^2)
 \le2\beta U\le2\beta^2\mathbf1.
\]
For each row of $\mathcal R$, weighted Cauchy--Schwarz and
$\mathcal R\mathbf1=\beta\mathbf1$ therefore yield
\[
 [\mathcal R(\gamma\sigma_{\pi,U})]^2
 \le(\mathcal R\mathbf1)
       \mathcal R(\gamma^2\sigma_{\pi,U}^2)
 \le2\beta^3\mathbf1.
\]
This proves \eqref{eq:unknown-control-optimal-resolvent}.
Finally, the standard-deviation triangle inequality gives
$\sigma_{\pi,\bar v}\le\sigma_{\pi,U}+e\mathbf1$,
which proves \eqref{eq:unknown-control-approx-resolvent}.
\end{proof}

\begin{proof}[Proof of Proposition~\ref{prop:unknown-control-bpe}]
Fix an initial query $(x_0,z_0)$.
Condition on any realization of the offline data satisfying
\eqref{eq:unknown-control-offline-event}, and set
$e=\|U-\bar v\|_\infty\le\eps/4$.
All probabilities and expectations over query samples in what
follows are conditional on these offline data.

\paragraph{Controlling the mass of unsuccessful local estimates.}
Because $q_{\bar v}(x,a,z)\in[0,\beta]$, we always have
$0\le D(x,z)\le\beta$. Lemma~\ref{lem:unknown-control-selected-variance}
therefore implies the samplewise inequality
\[
 D(x,z)\le
 3\sqrt{t_q}\,s_{\widehat\Pi(x,z)}(x,z)
 +16\beta t_q+\beta\mathbf1_{\mathcal E_{x,z}^c}.
\]
For the random cached policy, define the random vectors
\[
 d(x)=\E_{Z\sim\mu}D(x,Z),\qquad
 w(x)=\beta\E_{Z\sim\mu}\mathbf1_{\mathcal E_{x,Z}^c}.
\]
Here the context average is taken with the query-sample table fixed.
Averaging also over that table gives
$\E_{\rm query}[w(x)]\le\beta\rho_q$ for each $x$.
Thus Markov's inequality and a union bound over controlled states give
\begin{equation}
 \Pr\!\left\{\|w\|_\infty>\frac{\eps}{100\beta}\right\}
 \le\frac{100S\beta^2\rho_q}{\eps}
 =\delta/3.
 \label{eq:unknown-control-bad-mass}
\end{equation}
At the fixed initial pair,
$\Pr(\mathcal E_{x_0,z_0}^c)\le\rho_q\le\delta/3$.
These events need not be independent.

\paragraph{Propagating local losses along the cached policy.}
Fix a realization of the entire query-sample table for which
$\|w\|_\infty\le\eps/(100\beta)$ and $\mathcal E_{x_0,z_0}$ holds.
The resulting policy is now deterministic. Applying the law of
total variance and Jensen's inequality to an independent
$Z\sim\mu$ and successor draw gives
\[
 \E_{Z\sim\mu}s_{\widehat\Pi(x,Z)}(x,Z)
 \le\gamma\sigma_{\widehat\Pi,\bar v}(x).
\]
Indeed, the conditional successor law given $Z$ is
$P(\cdot\mid x,\widehat\Pi(x,Z),Z)$, and its marginal law is
$\bar P^{\widehat\Pi}(x,\cdot)$.
It follows that
\[
 d\le3\sqrt{t_q}\,\gamma\sigma_{\widehat\Pi,\bar v}
      +16\beta t_q\mathbf1+\frac{\eps}{100\beta}\mathbf1.
\]
Applying Lemma~\ref{lem:unknown-control-resolvent} gives
\begin{equation}
 \|\mathcal R_{\widehat\Pi}d\|_\infty
 \le3\sqrt{t_q}\left(\sqrt{2\beta^3}+\gamma\beta e\right)
       +16\beta^2t_q+\eps/100.
 \label{eq:unknown-control-propagated-loss}
\end{equation}

\paragraph{From the subsolution to the pointwise value guarantee.}
By definition of $d$ and the offline subsolution property,
\[
 T^{\widehat\Pi}\bar v=T\bar v-d\ge\bar v-d.
\]
The fixed-policy Bellman equation yields the exact identity
\[
 \bar V^{\widehat\Pi}-\bar v
 =\mathcal R_{\widehat\Pi}
       (T^{\widehat\Pi}\bar v-\bar v)
 \ge-\mathcal R_{\widehat\Pi}d.
\]
Using this inequality for the continuation value at $(x_0,z_0)$,
\begin{align*}
 V^\star(x_0,z_0)-V^{\widehat\Pi}(x_0,z_0)
 &\le\gamma e+D(x_0,z_0)
          +\gamma\|\mathcal R_{\widehat\Pi}d\|_\infty\\
 &\le e+3\beta\sqrt{t_q}+16\beta t_q\\
 &\quad+3\sqrt{t_q}\left(\sqrt{2\beta^3}+\beta e\right)
          +16\beta^2t_q+\eps/100.
\end{align*}
The second inequality uses $s_a(x,z)\le\beta$ and the local event
at the initial pair. Since
$t_q\le\eps^2/(4096\beta^3)$, $e\le\eps/4$,
$\eps\le1$, and $\beta\ge2$, the last expression is at most
\[
 \eps\left[
 \frac14+\frac3{64}\left(1+\sqrt2+\frac14\right)
       +\frac1{128}+\frac1{100}\right]
 <\eps/2<\eps.
\]
Adding the probabilities of offline failure, failure of
\eqref{eq:unknown-control-bad-mass}, and
$\mathcal E_{x_0,z_0}^c$ gives a total of at most $\delta$.
This proves the stated pointwise PAC guarantee.
\end{proof}

The uniform offline argument takes a net over continuation vectors in
$\mathbb R^S$. The query analysis takes a union over actions and controlled
states, while averaging local failure indicators over contexts.
Accordingly, neither budget contains any dependence on $|\Zcal|$,
including in its logarithmic factors. The BPE probability statement
holds for each fixed initial pair; it does not assert simultaneous
success at every context.
\subsubsection{Policy Evaluation}
\label{sec:unknown}

We prove the PE guarantee of Theorem~\ref{thm:main-pe-Punknown}.
Fix a deterministic stationary policy
$\pi:\Xcal\times\Zcal\to\Acal$, supplied as input.
Write $S=|\Xcal|$, $\beta=(1-\gamma)^{-1}$, and
$U=\bar V^\pi$. Rewards are known and belong to $[0,1]$.
We assume $\eps\in(0,1]$ and $\delta\in(0,1)$.

The proof reduces offline PE to the known-dynamics BVE problem already
solved in Theorem~\ref{thm:main-ub-mu-unknown}. This reduction reproduces
the sampled \textsc{HalfErr-PE} updates exactly. Thus the
variance-reduced argument, based on
\citet[Algorithms~1--2]{sidford2019nearoptimaltimesamplecomplexities},
can be reused without a second concentration and contraction analysis.

\paragraph{Offline upper bound}

\begin{theorem}[Offline PE with unknown $\mu$ and $P$]
\label{thm:pe-ub-pe}
There is an offline algorithm returning
$\hat{\bar V}^\pi\in[0,\beta]^S$ such that
\begin{equation}
 \Pr\!\left\{
 0\le\hat{\bar V}^\pi\le T^\pi\hat{\bar V}^\pi,\quad
 \|U-\hat{\bar V}^\pi\|_\infty\le\eps/4
 \right\}\ge1-\delta/3.
 \label{eq:pe-offline-success-v2}
\end{equation}
For a universal constant $C$, its oracle counts satisfy
\begin{equation}
 n_\mu=N,\qquad n_P=SN,\qquad
 N\le C\beta^3\bigl(\eps^{-2}+K\bigr)\lambda,
 \label{eq:pe-offline-count-v2}
\end{equation}
where
\[
 K=\left\lceil\log_2\frac{4\beta}{\eps}\right\rceil,\qquad
 R=\lceil\beta\log4\rceil,\qquad
 \lambda=\log\frac{96SK(R+1)}{\delta}.
\]
Consequently,
$n_{\rm learn}=n_\mu+n_P
=\widetilde O(S\beta^3/\eps^2)$, with no dependence on
$|\Zcal|$, including in logarithmic factors.
\end{theorem}

\paragraph{An auxiliary model with known transitions.}
Introduce the context space
\[
 \widetilde{\Zcal}=\Zcal\times\Xcal^{\Xcal}.
\]
An auxiliary context is $W=(Z,F)$, where the random map
$F:\Xcal\to\Xcal$ contains one successor for each controlled state.
Its distribution $\widetilde\mu$ is specified by
\begin{equation}
 Z\sim\mu,\qquad
 F(x)\mid Z\sim P(\cdot\mid x,\pi(x,Z),Z),
 \quad x\in\Xcal,
 \label{eq:pe-auxiliary-sampler}
\end{equation}
with the successors conditionally independent given $Z$.
Independent repetitions generate i.i.d.\ auxiliary contexts.
Each repetition costs one context call and $S$ transition calls.

The auxiliary MDP has controlled state space $\Xcal$, one action
$a_0$, discount factor $\gamma$, and known functions
\[
 \widetilde r(x,a_0,(z,f))=r(x,\pi(x,z),z),\qquad
 \widetilde P(x'\mid x,a_0,(z,f))=\mathbf1_{\{x'=f(x)\}}.
\]
In particular, evaluating $\widetilde P$ requires only reading $f(x)$.
The unknown original kernel enters the distribution of $W$,
which is accessible through \eqref{eq:pe-auxiliary-sampler}.
The only auxiliary action implements the prescribed original action
$\pi(x,z)$, so observing $f$ introduces no additional decision.

For every continuation vector $\bar v$, the auxiliary averaged
Bellman operator satisfies
\begin{align}
 (\widetilde T\bar v)(x)
 &=\E_{W\sim\widetilde\mu}
      [r(x,\pi(x,Z),Z)+\gamma\bar v(F(x))]\notag\\
 &=\E_{Z\sim\mu}\left[
      r(x,\pi(x,Z),Z)
      +\gamma\sum_{x'}P(x'\mid x,\pi(x,Z),Z)\bar v(x')
    \right]\notag\\
 &=(T^\pi\bar v)(x).
 \label{eq:pe-operator-identification}
\end{align}
Both operators are $\gamma$-contractions. Hence the auxiliary
averaged optimal value is exactly
\begin{equation}
 \widetilde{\bar V}^{\star}=\bar V^\pi=U.
 \label{eq:pe-fixed-point-identification}
\end{equation}
The auxiliary context space is never enumerated.

\paragraph{Identification with the sampled PE updates.}
The known-dynamics routine applied to an auxiliary context $W_j$
computes the anchor observation
\[
 Y_j(x)=r(x,\pi(x,Z_j),Z_j)+\gamma\bar v^{(0)}(F_j(x)).
\]
When the same auxiliary context is used to evaluate both terms of
an inner correction, their difference is
\[
 \widetilde g_{x,\bar v^{(i-1)}}(W_j^{(i)})
 -\widetilde g_{x,\bar v^{(0)}}(W_j^{(i)})
 =\gamma\bigl[
      \bar v^{(i-1)}(F_j^{(i)}(x))
      -\bar v^{(0)}(F_j^{(i)}(x))
    \bigr],
\]
where
$\widetilde g_{x,\bar v}(z,f)
=r(x,\pi(x,z),z)+\gamma\bar v(f(x))$.
These are precisely the anchor and correction observations in the
main's PE algorithm. The same successor is used in both terms.
All batches are fresh across inner iterations and stages;
each sampled context is shared across controlled states.
The empirical mean, empirical variance, and downward corrections are
those of the known-dynamics \textsc{HalfErr} routine.

\begin{proof}[Proof of Theorem~\ref{thm:pe-ub-pe}]
By \eqref{eq:pe-operator-identification}--%
\eqref{eq:pe-fixed-point-identification}, we may apply the offline
part of Theorem~\ref{thm:main-ub-mu-unknown} to the auxiliary model,
with accuracy $\eps/4$ and failure probability $\delta/3$.
Its guarantee gives \eqref{eq:pe-offline-success-v2}, and every
auxiliary context costs exactly $S+1$ original oracle calls.

We detail the schedule to match the PE implementation and eliminate
any implicit choice of logarithmic factors.
Set $\zeta=\delta/(3K)$ and $u_k=\beta2^{-k}$.
At stage $k=0,\ldots,K-1$, use
\begin{equation}
 n_{1,k}=
 \left\lceil c_1(1\vee u_k)\beta^3u_k^{-2}\lambda\right\rceil,
 \qquad
 n_2=\lceil c_2\beta^2\lambda\rceil,
 \qquad
 \alpha_k=\lambda/n_{1,k}.
 \label{eq:pe-stage-budgets-v2}
\end{equation}
For example, $c_1\ge2^{18}$ and $c_2\ge512$ suffice for the
corrected known-dynamics routine.
For every $u\in(0,\beta]$,
\begin{equation}
 \frac{\beta^3}{u^2}+\beta^2
 \le2(1\vee u)\frac{\beta^3}{u^2}.
 \label{eq:pe-coarse-radius-budget}
\end{equation}
Thus \eqref{eq:pe-stage-budgets-v2} meets the sufficient anchor budget
in \eqref{eq:known-halferr-parameters}, after a universal adjustment
of its constant.
Also $\lambda\ge\log(16SR/\zeta)$.
The proof of Proposition~\ref{prop:halving} therefore applies
unchanged, with $T^\pi$ identified as $\widetilde T$.
In particular, both bounds on $\alpha_k$ used in that proof hold:
\[
 \alpha_k\le\frac{u_k^2}{c_1\beta^3},
 \qquad
 \alpha_k\le\frac1{c_1\beta^2}.
\]

Initialize $\bar v_0=0$ and run the $K$ halving stages with these
parameters. Use the coordinatewise update
\[
 \bar v^{(i)}
 =\min\{\beta\mathbf1,\,
         \max\{\bar v^{(i-1)},w+\Delta^{(i)}\}\}.
\]
Upper clipping makes all outputs belong to $[0,\beta]^S$ and is
inactive whenever the halving guarantee holds.
By the conditional composition argument already established in
Theorem~\ref{thm:main-ub-mu-unknown}, all stages succeed with
probability at least $1-K\zeta=1-\delta/3$.
Their final radius is $\beta2^{-K}\le\eps/4$.

The number of auxiliary contexts is exactly
\[
 N=\sum_{k=0}^{K-1}(n_{1,k}+Rn_2).
\]
Since $R=O(\beta)$, the stage cost is at most
\[
 C\beta^3\lambda
 \left(1+\frac{1\vee u_k}{u_k^2}\right).
\]
The geometric radii give
\[
 \sum_{k:u_k>1}\frac1{u_k}\le2,\qquad
 \sum_{k:u_k\le1}\frac1{u_k^2}
 \le\sum_{k=0}^{K-1}\frac1{u_k^2}
 =\frac{4^K-1}{3\beta^2}
 \le\frac{64}{3\eps^2}.
\]
Summing the stage costs proves \eqref{eq:pe-offline-count-v2}.
The theorem for known dynamics has no dependence on the cardinality
of its context space, so the enlarged space introduces no additional
sample factor.
\end{proof}

\paragraph{The joint variance bound is already covered by the reduction.}
For consistency with the main, recall
\[
 v^\pi(x)=\operatorname{Var}_{Z\sim\mu}
       (g^\pi_{x,U}(Z)),\qquad
 \sigma^\pi(x)=\gamma^2\E_{Z\sim\mu}
       \operatorname{Var}_{P(\cdot\mid x,\pi(x,Z),Z)}(U(X')).
\]
Here $\sigma^\pi$ denotes a variance contribution, as in the main.

\begin{corollary}[Joint contextual and transition variance]
\label{cor:V-pi-bound-pe}
The averaged kernel $\bar P^\pi$ and the two variance contributions
above satisfy
\begin{align}
 (I-\gamma^2\bar P^\pi)^{-1}(v^\pi+\sigma^\pi)
 &\le\beta^2\mathbf1,
 \label{eq:pe-joint-variance-v2}\\
 \left\|(I-\gamma\bar P^\pi)^{-1}
              \sqrt{v^\pi+\sigma^\pi}\right\|_\infty
 &\le\sqrt2\,\beta^{3/2}.
 \label{eq:pe-joint-resolvent-v2}
\end{align}
\end{corollary}

\begin{proof}
By the law of total variance, the contextual variance in the
auxiliary model is
\[
 \operatorname{Var}_{W\sim\widetilde\mu}
       [r(x,\pi(x,Z),Z)+\gamma U(F(x))]
 =v^\pi(x)+\sigma^\pi(x).
\]
Its averaged transition matrix is $\bar P^\pi$.
Apply Lemma~\ref{lem:variance-bound} and
Corollary~\ref{cor:V-star-bound} to this auxiliary model.
\end{proof}

\paragraph{Query upper bound and the PE guarantee}

\begin{proposition}[Pointwise PE upper bound]
\label{prop:pe-lift-unknown}
Combining the offline procedure of Theorem~\ref{thm:pe-ub-pe}
with the query mechanism below achieves the PE objective with
\[
 n_{\rm learn}=\widetilde O(S\beta^3/\eps^2),\qquad
 m_{\rm query}
 =\left\lceil\frac{2\beta^2}{\eps^2}
                   \log\frac6\delta\right\rceil.
\]
There is no dependence on $|\Acal|$ or $|\Zcal|$ in either budget,
including in logarithmic factors.
\end{proposition}

\begin{proof}
At a fixed query $(x,z)$, set $a=\pi(x,z)$.
Draw $m=m_{\rm query}$ fresh successors from $P(\cdot\mid x,a,z)$,
independently of the offline data, and return
\[
 \widehat{\Vcal}^\pi(x,z)
 =r(x,a,z)+\frac{\gamma}{m}
       \sum_{j=1}^m\hat{\bar V}^\pi(X_j).
\]
This is the sampled value-lift argument of
Proposition~\ref{prop:unknown-control-bve}, applied to the single
prescribed action. Conditional on any offline output in $[0,\beta]^S$,
its sampling error is at most $\eps/2$ with probability at least
$1-\delta/3$.
On the intersection of this event with
\eqref{eq:pe-offline-success-v2}, the Bellman identity gives
\[
 |\widehat{\Vcal}^\pi(x,z)-V^\pi(x,z)|
 \le\eps/2+\gamma\|\hat{\bar V}^\pi-U\|_\infty
 \le3\eps/4<\eps.
\]
The total failure probability is at most $2\delta/3\le\delta$.
The query makes $m$ transition calls and no context calls.
This completes the proof of Theorem~\ref{thm:main-pe-Punknown}.
\end{proof}

\end{document}